\pdfoutput=1
\documentclass{article}

\usepackage[english]{babel}

\usepackage[letterpaper,top=2cm,bottom=2cm,left=3cm,right=3cm,marginparwidth=1.75cm]{geometry}

\usepackage{booktabs}
\usepackage{multirow}
\usepackage{amsmath}
\usepackage{amsthm}
\usepackage{graphicx}
\usepackage[colorlinks=true, allcolors=blue]{hyperref}
\usepackage{flushend, cuted}
\usepackage[title]{appendix}

\usepackage{amsfonts}
\usepackage{bbm}

\usepackage{amssymb}
\usepackage{float}
\usepackage{color,amsmath,amssymb, amsfonts,amstext,amsthm}

\usepackage{epsfig, graphicx, graphics}
\usepackage{longtable}
\usepackage{cite}
\usepackage{subfigure}

\usepackage{caption}
\usepackage{bm}
\usepackage{epstopdf}

\usepackage{color, amsmath,amssymb, amsfonts, amstext,latexsym}

\usepackage{amssymb, epsfig, amssymb, latexsym}

\usepackage{longtable}

\usepackage{authblk}

\usepackage[mathscr]{euscript}

\usepackage{color}

\renewcommand{\phi}{\varphi}

\newcommand{\sign}[1]{\mathrm{sgn}(#1)}

\newtheorem{definition}{Definition}

\newtheorem{remark}{Remark}

\newtheorem{proposition}{Proposition}

\title{An end-to-end deep learning approach for extracting stochastic dynamical systems with $\alpha$-stable L\'evy noise}

\author[a]{Cheng Fang \thanks{fangcheng1@hust.edu.cn}}
\author[a]{Yubin Lu \thanks{yubin\_lu@hust.edu.cn}}
\author[a]{Ting Gao \thanks{Corresponding author: tgao0716@hust.edu.cn }}
\author[b]{Jinqiao Duan \thanks{duan@iit.edu}}
\affil[a]{School of Mathematics and Statistics \& Center for Mathematical Sciences, Huazhong University of Science and Technology, Wuhan 430074, China}
\affil[b]{Department of Applied Mathematics, College of Computing, Illinois Institute of Technology, Chicago, IL 60616, USA}

\begin{document}
\date{June 3, 2022}
\maketitle

\begin{abstract}
Recently, extracting data-driven governing laws of dynamical systems through deep learning frameworks has gained a lot of attention in various fields. Moreover, a growing amount of research work tends to transfer deterministic dynamical systems to stochastic dynamical systems, especially those driven by non-Gaussian multiplicative noise. However, lots of log-likelihood based algorithms that work well for Gaussian cases cannot be directly extended to non-Gaussian scenarios which could have high error and low convergence issues. In this work, we overcome some of these challenges and identify  stochastic dynamical systems driven by $\alpha$-stable L\'evy noise from only random pairwise data. Our innovations include: (1) designing a deep learning approach to learn both drift and diffusion coefficients for L\'evy induced noise with $\alpha$ across all values, (2) learning complex multiplicative noise without restrictions on small noise intensity, (3) proposing an end-to-end complete framework for stochastic systems identification under a general input data assumption, that is, $\alpha$-stable random variable. Finally, numerical experiments and comparisons with the non-local Kramers-Moyal formulas with moment generating function confirm the effectiveness of our method.  
\end{abstract}

\textbf{Stochastic differential equations are commonly employed to describe phenomena in many application fields. To better analyze the intrinsic dynamics of the system, it is important to get good approximations of the vector fields and noise types in the stochastic systems from real observation data.  This kind of data-driven problem is usually investigated in various optimization ways. However, most of these methods are limited to observation data under Gaussian noise, which cannot be applied to many real-world scenarios such as climate change, genetic transcription, finance crisis, etc., where the data contains jumping fluctuations or obeys distributions deviating from the general Gaussian assumption. To solve this problem, we construct a novel approach to extract stochastic governing laws from observation data under $\alpha$-stable L\'evy noise.  Our method is general, adaptable, and capable of dealing with a wide range of situations, especially large and multiplicative noise. Various experimental results are presented considering different types of vector fields and noise situations.}

\section{Introduction}
Stochastic dynamical systems, described by stochastic differential equations (SDEs), are widely used to describe various natural phenomena in Physics, Biology, Economics, Ecology, etc. When there is a lack of scientific understanding of complex phenomena or mathematical models based on the governing laws are too complex, the usual analytical process is difficult to work. Fortunately, with the development of observation technology and computing power, a great deal of valuable observation data can be obtained from the above situation. Further, a lot of data-driven problems arise in terms of identifying stochastic governing laws for different types of noises.  Therefore, it is important to investigate accurate and efficient methods to solve unknown coefficients or functions in complex SDE models.

The exploration of identifying stochastic dynamical systems usually focuses on models expressed by deterministic differential equations under Gaussian noise. For instance, Ruttor \cite{Ruttor2013ApproximateGP} applies Gaussian process prior over the drift coefficient and develops the Expectation-Maximization algorithm to deal with latent dynamics between observations. Dai \cite{Dai2020DetectingTM} leverages Kramers–Moyal formulas and Extended Sparse Identification algorithms for nonlinear dynamics to obtain SDEs coefficients and maximum likelihood transition pathways. Klus \cite{Klus2019DatadrivenAO} derives a data-driven method for the approximation of the Koopman operator which is appropriate to identify the drift and diffusion coefficients of stochastic differential equations from real data. Some variational approaches are also used to approximate the distribution over the unknown paths of the SDE conditioned on the observations and approximate the intractable likelihood of drift \cite{Opper2019VariationalIF}. There are also other data-driven methods for this purpose, including but not limited to Bayesian inference \cite{Garca2017NonparametricEO}, sparse identification \cite{Brunton2016DiscoveringGE} and so on. The emergence of deep learning, thanks to the recent fast improvement of computing power, has made great progress in many application fields such as computer vision, language modeling and signal processing. Since the major mathematical formulations of these problems are optimization, it is natural to deal with the inverse problems such as system identification by some deep learning framework, through which lots of research work has nowadays been carried out. For example, Dietrich \cite{Dietrich2021LearningES} approximates the drift and diffusivity functions in the effective SDE through effective stochastic ResNets \cite{He2016DeepRL}. Xu and Darve \cite{Xu2021SolvingIP} leverage a discriminative neural network for computing the statistical discrepancies which can learn the model’s unknown parameters and distributions that is inspired by GAN \cite{Goodfellow2014GenerativeAN}. Dridi \cite{Dridi2021LearningSD} proposes a novel approach where parameters of the unknown model are represented by a neural network integrated with SDE scheme. Ryder \cite{Ryder2018BlackBoxVI} uses variational inference to jointly learn the parameters and the diffusion paths, and then introduce a recurrent neural network to approximate the posterior for the diffusion paths conditional on the parameters. Neural ordinary differential equation (NODE) \cite{Chen2018NeuralOD} and its stochastic expansions (NSDE) (\cite{Li2020ScalableGF} \cite{Jia2019NeuralJS} \cite{Tzen2019NeuralSD} \cite{Norcliffe2021NeuralOP}) also strive to describe the evolution of the system for continuous time intervals which can be applied to many application fields.

However, due to the fact that real-world observation data can usually have various jumps or bursts, it is more suitable to model them as stochastic dynamic systems driven by non-Gaussian fluctuations, e.g., L\'evy flights. For example, L\'evy motions can be used to describe random fluctuations that appear in the oceanic fluid flows \cite{Woyczynski2001LvyPI}, gene networks \cite{Cai2017LvyNE}, biological evolution \cite{Jourdain2012LvyFI}, finance \cite{Nolan2003ModelingFD} and geophysical systems \cite{Yang2020TheTT}, etc.  All these indicate that stochastic dynamical systems with L\'evy motions are more appropriate to model real-world phenomena scientifically. Therefore, increasing amounts of research work on such kind of data-driven problems are rising recently. For example, through the non-local Kramers-Moyal formulas, 
Li \cite{Li2021ExtractingSD} analytically represents $\alpha$-stable L\'evy jump measures, drift coefficients, and diffusion coefficients using either the transition probability density or the sample paths, which can be achieved by normalizing flows or other basis function based machine learning methods \cite{Li2020ADA} \cite{Lu2021ExtractingSG} \cite{Li2021ExtractingGL}. Another way is to learn SDE’s coefficients through neural networks instead of learning of the corresponding nonlocal Fokker-Planck equations \cite{Chen2021SolvingIS}. In addition, generalizing the Koopman operator into non-Gaussian noise also allows the coefficients of the stochastic differential equations to be estimated \cite{Lu2020DiscoveringTP}.

In this present work, our goal is to explore an effective data-driven approach to learn stochastic dynamical systems under $\alpha$-stable L\'evy noise. Our work focuses heavily on the distribution characteristics of the data and includes three major contributions: First, we design a two-step hybrid neural network structure to learn both drift and diffusion coefficients under L\'evy induced noise with $\alpha$ across all values, while some methods like nonlocal Kramers-Moyal formulas with moment generating function can only handle the case when $\alpha$ is larger than 1; Second, we can learn complex multiplicative noise without any pre-requirements on small noise intensity, which is often the case for many existing methods; Third, we propose an end-to-end complete framework for stochastic systems identification under a comparatively general prior distribution assumption, that is, input data is supposed to be a $\alpha$-stable random variable.

The paper is organized as follows. In Section \ref{Problem setting}, we introduce some background knowledge including stochastic differential equations driven by $\alpha$-stable L\'evy motions and the corresponding deep learning based method directly extended from the Brownian motion case. After checking the challenging issues with extended formulas, we present our proposed algorithms in detail in Section \ref{Research framework}. Then, in Section \ref{experiments}, we give out representative experiment results of stochastic differential equations with additive or multiplicative $\alpha$-stable L\'evy noise and compare them with the non-local Kramers-Moyal method with moment generating function which has restrictions on additive noise and stability parameter. Under special assumptions, we approximate the drift and diffusion coefficients of a two-dimensional Maier-Stein model. Finally, we conclude the advantages and future work of this research paper. Moreover, many detailed explanations and results are presented in Appendices \ref{alpha rv}, \ref{drift trick} and \ref{other experiments}, as well as properties of $\alpha$-stable random variables, data preprocessing tricks ,richer experimental results and so on.

\section{Problem setting} \label{Problem setting}

\subsection{Stochastic differential equations driven by $\alpha$-stable L\'evy noise} \label{sde setting}
We consider a special but important class of L\'evy motions, $\alpha$-stable L\'evy motions. The notation $S_\alpha (\sigma, \beta, \gamma)$ represents an $\alpha$-stable random variable with four parameters: an index of stability $\alpha \in (0,2]$ also called the tail index, tail exponent or characteristic exponent, a skewness parameter $\beta \in [-1,1]$, a scale parameter $\sigma > 0$ and a location parameter $\gamma \in \mathbb{R}^1$. See Appendix \ref{alpha rv} for more details.

\begin{definition} \label{levy motion def}
Defined  a probability space $(\Omega, \mathcal{F}, \mathit{P})$, a symmetric $\alpha$-stable scalar L\'evy motion $L_t^{\alpha}$, with $0<\alpha<2$, is a stochastic process with the following properties:
\begin{itemize}
    \item $L_0^{\alpha}=0$, a.s.;
    \item $L_t^{\alpha}$ has independent increments;
    \item $L_t^{\alpha}-L_s^{\alpha} \sim S_\alpha ((t-s)^{\frac{1}{\alpha}}, 0, 0)$;
    \item $L_t^{\alpha}$ is stochastically continuous, i.e., for all $\delta > 0$ and for all $s \ge 0$
    \begin{equation*}
        \lim_{t \rightarrow s} \mathit{P}(|L_t^{\alpha}-L_s^{\alpha}| > \delta) = 0. 
    \end{equation*}
\end{itemize}
\end{definition}

The $\alpha$-stable L\'evy motions $L_t^{\alpha}$ in $\mathbb{R}^d$ can be similarly defined. In the L\'evy-Khintchine formula \cite{duan2015introduction}, $(b,\mathit{Q},\nu_\alpha)$ is the (generating) triplet for the $\alpha$-stable L\'evy motion $L_t^{\alpha}$ which represents drift vector, covariance matrix and jump measure, separately. Usually, we consider the pure jump case $(0, 0, \nu_\alpha)$.

Stochastic differential equations (SDEs) are differential equations involving noise. Based on L\'evy-It\^o decomposition \cite{duan2015introduction}, L\'evy motion generally includes three parts, that is, a drift term, a diffusion term induced by Brownian motion and another diffusion term with pure jump noise. In this paper, we consider autonomous stochastic differential equations (non-autonomous equations can be treated as autonomous systems with one additional time dimension) with $\alpha$-stable L\'evy motions and build algorithms to discover their corresponding stochastic governing laws. Here we consider the following SDE:
\begin{equation}
    dX_t=f(X_t)dt+g(X_t)dL^{\alpha}_t, \label{LDE}
\end{equation}
where $X_t\in \mathbb{R}^d$ and $f(X_t)\in \mathbb{R}^d$ is the drift coefficient, $g(X_t)\in \mathbb{R}^d\times\mathbb{R}^d$ is the noise intensity, and $dL^{\alpha}_t=[dL^{\alpha}_1(t),...,dL^{\alpha}_d(t)]^T$ is composed of $d$ mutually independent one-dimensional symmetric $\alpha$-stable L\'evy motions with triplet $(0,0,\nu_\alpha)$. In the triplet, the jump measure $\nu_\alpha(dx) = C(1,\alpha) \parallel x \parallel^{-1-\alpha}dx$ for $x \in \mathbb{R}^1\setminus \{0\}$, and the intensity constant
\begin{equation*}
    C(1, \alpha) = \dfrac{\alpha\Gamma((1 +\alpha)/2)}{2^{1-\alpha}\pi^{1/2}\Gamma(1-\alpha/2)}.
\end{equation*}
For this $\alpha$-stable L\'evy motion, we see that its components $L^{\alpha}_i(t) \sim S_\alpha(t^{\frac{1}{\alpha}},0,0)$,\ $i=1,\ 2,\ \dots,\ d$. 

Note that for the multi-dimensional case, we only consider the diffusion coefficient $g(X_t)$ to be a $d \times d$ square diagonal matrix under elementary transformation, with non-negative elements of the main diagonal on the domain of $X_t$, that is, $g_{ii}(X_t) \ge 0$ for $i=1,2,\dots,d$. Now for more general case when $m$ denotes the dimension of L\'evy motion and $m \neq d$, we have all the off-diagonal elements of $d \times m$ matrix $g(X_t)$ being zero under elementary transformation. We can add or remove zero vectors to construct a $d \times d$ matrix. Some reasons for this assumption can be found in Section \ref{multi_dim}.

\subsection{Challenges of traditional log-likelihood approach} \label{challenge}

The Euler-Maruyama scheme is a method to approximate (\ref{LDE}) over a small time interval $h > 0$:
\begin{equation}
    X_n= X_{n-1} + hf(X_{n-1}) + g(X_{n-1})L_h^{\alpha},\quad n=1,\ 2,\ \dots, \label{E-M LDE}
\end{equation}
where $L_h^{\alpha}$ is a $d$-dimensional random vector, the components of it are mutually independent and $\alpha$-stable distributed, i.e., $L^{\alpha}_i(h) \sim S_\alpha(h^{\frac{1}{\alpha}},0,0)$, $i=1,\ 2,\ \dots,\ d$. This method can be derived from stochastic Taylor expansion. The convergence and convergence order of (\ref{E-M LDE}) for $h \rightarrow 0$ have been studied at length.

Assuming we can only use a set of $N$ snapshots $D = \{(x_1^{(k)},x_0^{(k)},h^{(k)})\}_{k=1}^N$ where $x_0^{(k)}$ are points scattered in the state space of (\ref{LDE}) and the value of $x_1^{(k)}$ results from the evolution of (\ref{LDE}) under a small time-step $h^{(k)} > 0$, which starts at $x_0^{(k)}$.
~\\

\textbf{Log-likelihood method for SDEs driven by Brownian motions and $\alpha$-stable L\'evy motions:}

Likelihood estimation in combination with the Gaussian distribution ($\alpha=2$) is used in many variational and generative approaches (\cite{Goodfellow2014GenerativeAN} \cite{Kingma2014AutoEncodingVB} \cite{Li2020ScalableGF} \cite{Opper2019VariationalIF} \cite{Dai2020DetectingTM} \cite{Ryder2018BlackBoxVI}). Based on the Euler-Maruyama discretization method (\ref{E-M LDE}), we can construct a loss function for training two neural networks to approximate $f$ and $g$ in (\ref{E-M LDE}), denoted as $f_\theta$ and $g_\theta$. Taking advantage of the properties of the Gaussian distribution, the probability of $x_1$ conditioned on $x_0$ and $h$ is given by:
\begin{equation}
    x_1 \sim  \mathcal{N}(x_0+hf(x_0), hg(x_0)^2). \label{x1 gaussian distribution}
\end{equation}

Now we utilize the training data $D$ in terms of triples $(x_1^{(k)},x_0^{(k)},h^{(k)})$ to approximate drift $f$ and diffusion $g$. By defining the probability density $p_\theta$ of the Gaussian distribution (\ref{x1 gaussian distribution}), we could obtain the neural networks $f_\theta$ and $g_\theta$ through maximizing log-likelihood of the data $D$ under the assumption in equation (\ref{x1 gaussian distribution}):
\begin{equation}
    \theta := \arg \max_{\hat{\theta}} \mathbb{E} [\log p_{\hat{\theta}}(x_1|x_0,h)] \approx \arg \max_{\hat{\theta}}\frac{1}{N}\sum_{k=1}^{N}\log p_{\hat{\theta}}(x_1^{(k)}|x_0^{(k)},h^{(k)}). \label{theta def} 
\end{equation}
Applying the logarithm of the probability density function of the Gaussian distribution combined with the parameters represented by neural networks, we could construct the loss function as
\begin{equation} \label{mle gaussian loss}
    \mathcal{L}(\theta|x_1,x_0,h):= \frac{1}{N}\sum_{k=1}^{N} \frac{(x_1^{(k)} - (x_0^{(k)}+hf_\theta(x_0^{(k)})))^2}{2hg_\theta(x_0^{(k)})^2} + \frac{1}{2N} \sum_{k=1}^{N} \log|hg_\theta(x_0^{(k)})^2| + \frac{1}{2}\log (2\pi).
\end{equation}
Minimizing $\mathcal{L}$ in (\ref{mle gaussian loss}) over the data $D$ is equivalent to maximization of the log marginal likelihood (\ref{theta def}). After training, the neural network outputs $\hat{f}_\theta$, $\hat{g}_\theta$ are what we want. 

Now we can extend the above procedure directly to the stochastic differential equation driven by non-Gaussian noise (\ref{LDE}), for example, consider the dataset generated by the following equation using the discretization method (\ref{E-M LDE}):
\begin{equation} \label{experiment of only mle}
    dX_t= (-X_t+1)dt+0.1dL^{\alpha}_t
\end{equation}
where $\alpha=1.5$, $X_t \in \mathbb{R}^1$. By applying the negative log-likelihood function as loss function:
\begin{equation} \label{mle failure levy loss}
    \mathcal{L}(\theta|x_1,x_0,h):= -\frac{1}{N}\sum_{k=1}^{N}\log p_{\theta}(x_1^{(k)}|x_0^{(k)},h^{(k)}),
\end{equation}
where $p_{\theta}$ represents the probability density function of the $\alpha$-stable random variable whose concrete form and derivation can be found in Section \ref{levy structure}. The training results are shown in Figure \ref{fig:only mle loss}.

\begin{figure}[H]
  \centering
  \subfigure[The approximation results of SDE (\ref{experiment of only mle}), (left) true and approximated drift coefficients, (right) true and approximated diffusivity coefficients]{
    \label{fig:subfig:onefunction} 
    \includegraphics[scale=0.48]{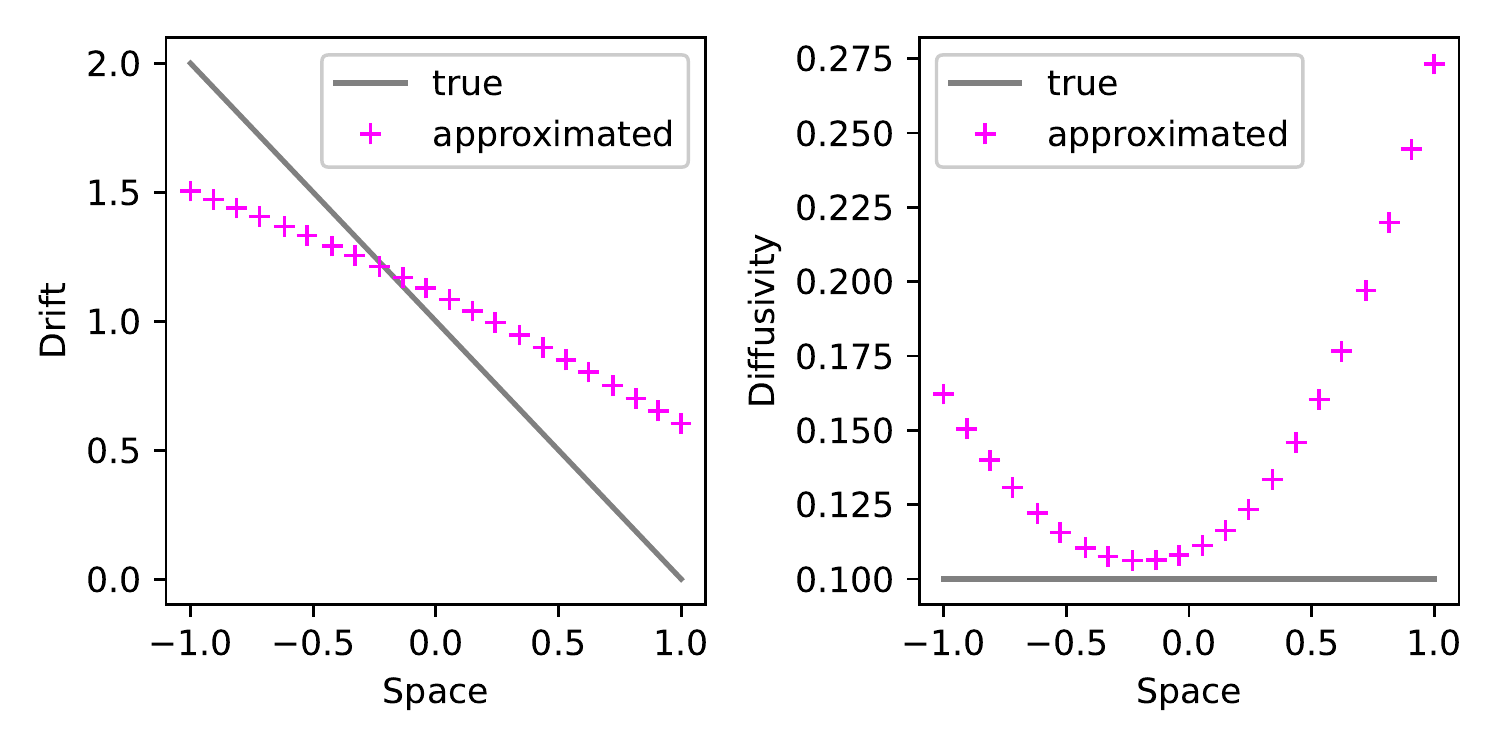}}
  \hspace{0.1in} 
  \subfigure[Loss function]{
    \label{fig:subfig:onefunction loss} 
    \includegraphics[scale=0.4]{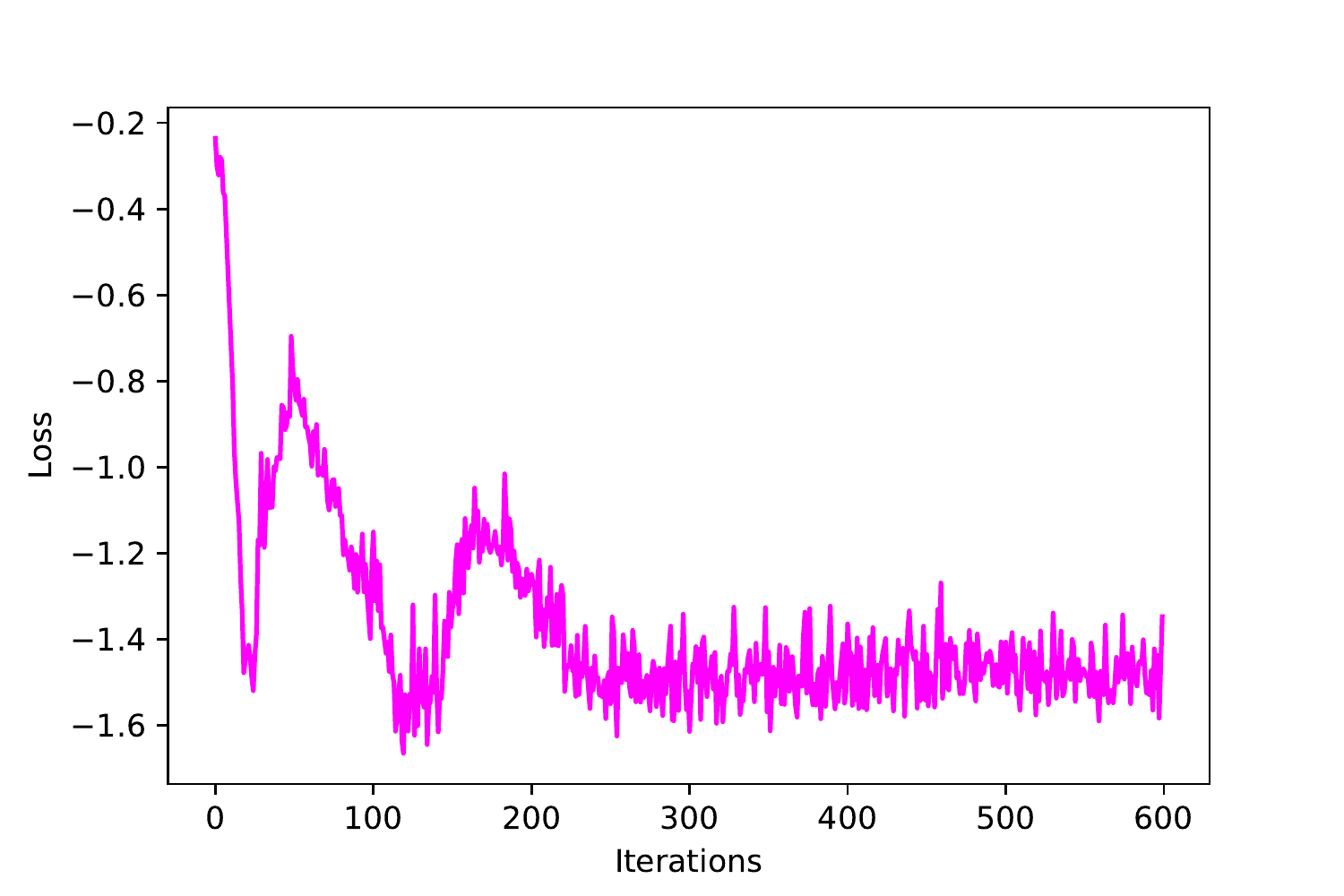}}
  \caption{The estimated results of drift and diffusion coefficients, and corresponding loss. (a) Comparison of ground truth and approximation of drift and diffusion coefficients. (b) Loss over iterations shows training convergence.}
  \label{fig:only mle loss} 
\end{figure}

It is shown in Figure \ref{fig:only mle loss} that even the training loss is decreasing and convergent, the approximation of drift and diffusion coefficients is still far from accurate. Therefore, log-likelihood function as a loss function is not enough for the identification of $\alpha$-stable L\'evy noise driven stochastic differential equations. Now we present our framework in details in the next section for solving this issue.

\section{Research framework} \label{Research framework}

In this section,  we investigate how to discover stochastic dynamics driven by $\alpha$-stable L\'evy noise. Here we propose a deep learning algorithm to extract the stochastic differential equations as the form (\ref{LDE}) from samples. It is important to note that our research focuses on both
additive and multiplicative noise, without limited requirements on small noise intensity. Before introducing our methodology, we give the workflow of our constructed framework (Figure \ref{workflow}).

\begin{figure}[H] 
\centering
\includegraphics[width=14cm]{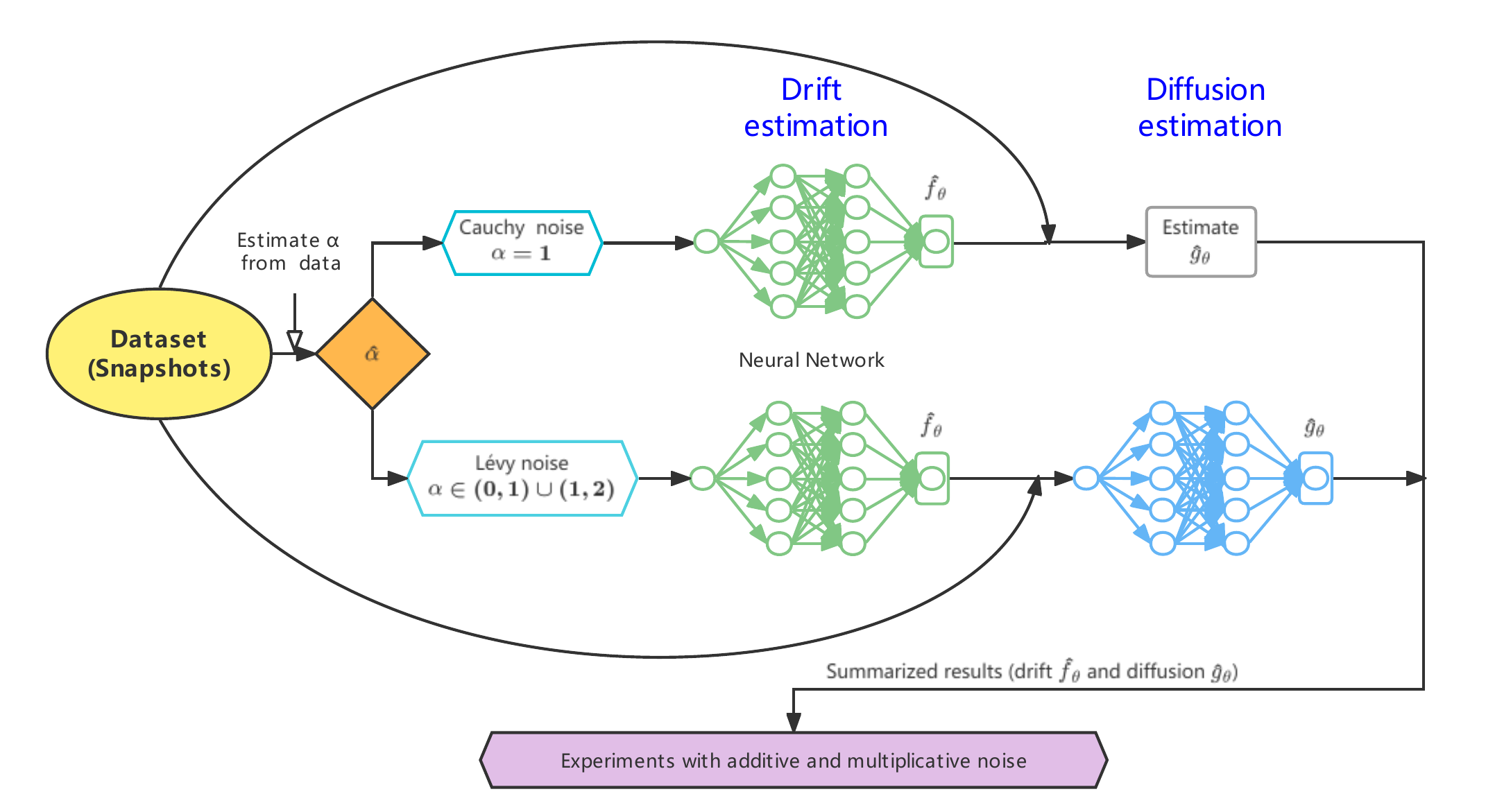}
\caption{The workflow of our proposed framework. The estimation of stability parameter $\alpha$ is briefly introduced in Section \ref{MCMC moment generating}. After determining whether it is Cauchy noise ($\alpha = 1$), different structures are used to identify the dynamic behavior, see details in Section \ref{proposed method}. Experiments with stochastic differential equations driven by additive or multiplicative noise are carried out in Section \ref{experiments}.}
\label{workflow}
\end{figure}

\subsection{Estimation of stability parameter $\alpha$} \label{MCMC moment generating}

In Section \ref{sde setting}, we know that $\alpha$-stable random variables have stability parameter $\alpha$. The stability parameter estimation is not a new research topic and has been widely studied. Therefore, a couple of techniques have been investigated for stability parameter estimation, such as quantile method \cite{mcculloch1986simple}, characteristic function method \cite{Ilow1998ApplicationsOT}, maximum likelihood method \cite{Nolan2001MaximumLE}, extreme value method, fractional lower order moment method \cite{Bhaskar2010VideoFD}, method of log-cumulant \cite{Nicolas2011IntroductionTS}, characteristic function based analytical approach \cite{Bibalan2017CharacteristicFB}, Markov chain Monte Carlo with Metropolis-Hastings algorithm\cite{Hao2011ParameterEO}, moment generating function method \cite{Nolan2013MultivariateEC}\cite{Li2021ExtractingSD}. 

We attempt to estimate $\alpha$ using Markov chain Monte Carlo with Metropolis-Hastings algorithm and moment generating function method: when the stability parameter equals to 1.5, i.e., $\alpha = 1.5$, Markov chain Monte Carlo with Metropolis-Hastings algorithm (MCMC with MH) shows that $\alpha$ converges to about $1.4930$ (Figure \ref{fig:mcmc with mh}); and the result of direct calculation of the moment generating function is $1.5161$. Due to the validity of the above methods, the stability parameter $\alpha$ is assumed to be known in subsequent studies.

\begin{figure}[H]
  \centering
  \subfigure[Iteration results]{
    \label{fig:subfig:mcmc with mh} 
    \includegraphics[scale=0.48]{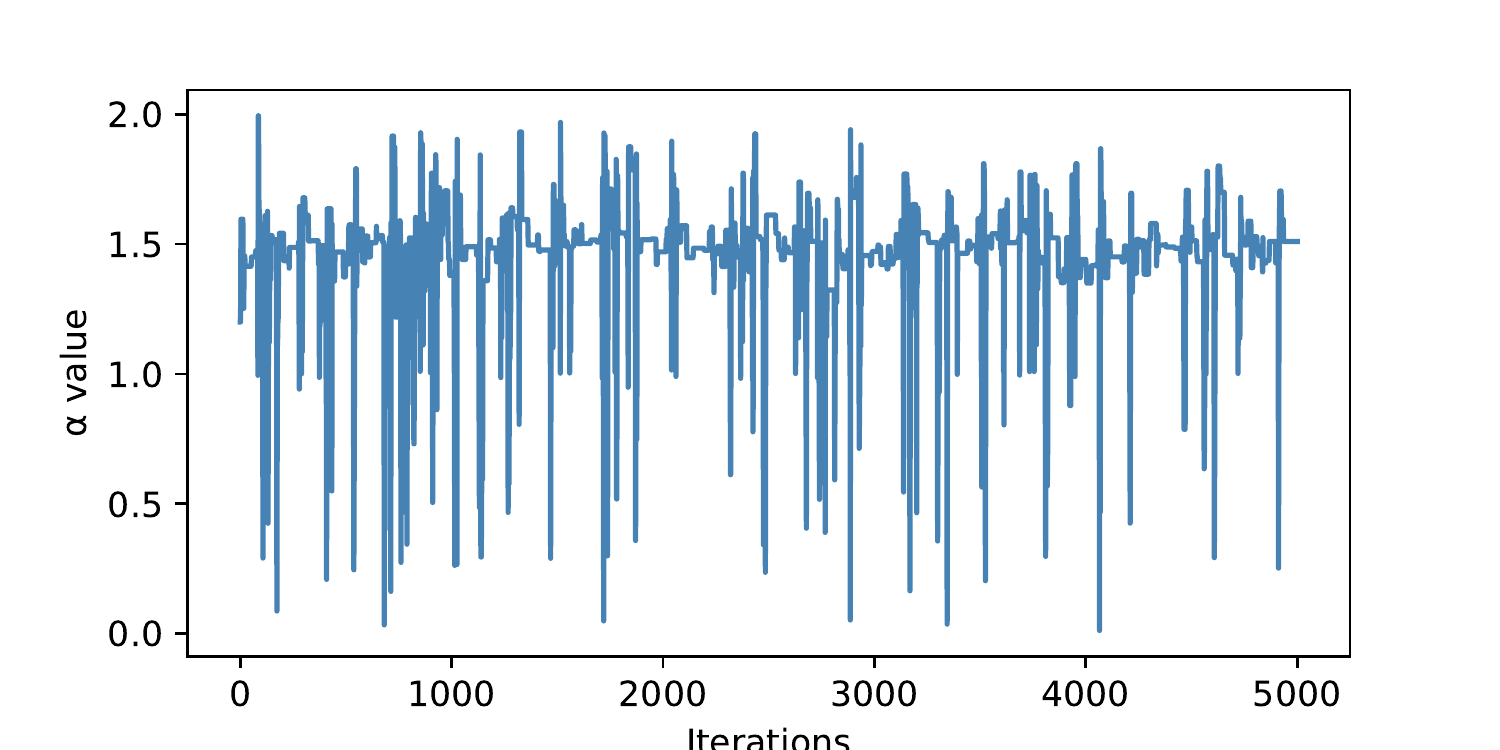}}
  \hspace{0.1in} 
  \subfigure[Histogram of the $\alpha$ value frequency number]{
    \label{fig:subfig:histogram} 
    \includegraphics[scale=0.48]{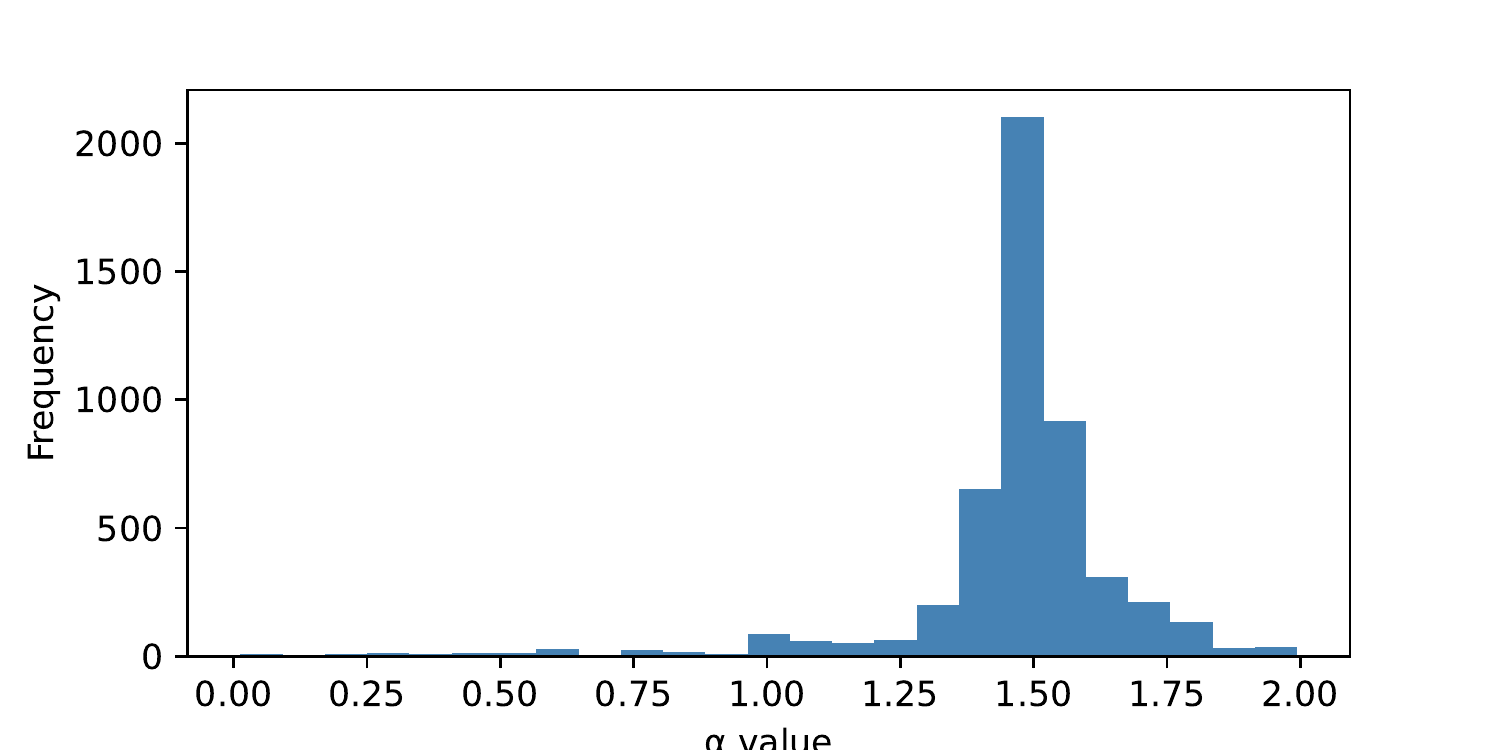}}
  \caption{For the true $\alpha = 1.5$, the iterative result of  Markov chain Monte Carlo with Metropolis-Hastings algorithm (a) and a histogram of the $\alpha$ frequency number (b). Both iteration and histogram show that the $\alpha$ value is around $1.5$.}
  \label{fig:mcmc with mh} 
\end{figure}

\subsection{Challenge issues explanation} \label{Log-likelihood estimation analysis }

Here we continue to use the data structure mentioned in Section \ref{challenge}, i.e., a set of $N$ snapshots $D = \{(x_1^{(k)},x_0^{(k)},h^{(k)})\}_{k=1}^N$. Since we are thinking about autonomous systems, these samples can be viewed as starting at the same time in the state space and passing through the same time-step $h$, or they can be viewed as taking data pairs from a long trajectory $\{x_{t_i}\}$ of (\ref{LDE}) with sample frequency $h_i > 0$, that is, $t_{i+1}=t_i+h_i$. Note that, the step size $h^{(k)}$ is defined for each snapshot,  so it may have different values for every index $k$ and every task. The main research is to discover drift $f$ and diffusion (diffusivity or noise intensity) $g$ through two neural networks $f_\theta: \mathbb{R}^d \rightarrow \mathbb{R}^d$ and $g_\theta:\mathbb{R}^d \rightarrow \mathbb{R}^d\times\mathbb{R}^d$, parameterized by their weights and bias $\theta$, only from the data in $D$. We restrict the discussion to the case $d = 1$ for our main work in Section \ref{proposed method} and multi-dimensional cases could be explained and discussed in Section \ref{multi_dim}. Our initial idea is to utilize log-likelihood estimation to construct the loss function, that is, minimizing loss function implies maximization of the log-likelihood function. However, as can be seen at the end of Section \ref{challenge}, log-likelihood alone is not enough to serve as the loss function of neural networks, and we give the following explanation. 

~\\

\textbf{Log-likelihood function in a dilemma}:

\emph{First}, we discuss the problems encountered in the case of maximal likelihood estimation in the face of the Cauchy distribution, i.e., stable parameter $\alpha=1$. Maximum likelihood can also be used to estimate the location parameter $\gamma$ and the scale parameter $\sigma$ in Table \ref{tab:closed form}. If a set of i.i.d. samples of size $N$ is taken from a Cauchy distribution ($S_1(\sigma,0,\gamma)$), the log-likelihood function is:
\begin{equation}
    \mathcal{L}(z_1,\dots,z_N|\sigma,\gamma)=\frac{1}{N}\sum_{k=1}^N p_{\alpha=1}(z_k|\sigma,\gamma)=-\log(\sigma \pi)-\frac{1}{N} \sum_{k=1}^N \log \Big(1+\Big(\frac{z_k-\gamma}{\sigma}\Big)^2\Big).
\end{equation}
Maximizing the log-likelihood function with respect to $\gamma$ and $\sigma$ by taking the first derivative produces the following system of equations:
\begin{equation}
    \begin{split}
        \frac{d\mathcal{L}}{d\gamma} &= \sum_{k=1}^N \frac{2(z_k-\gamma)}{\sigma^2+(z_k-\gamma)^2} = 0,\\
        \frac{d\mathcal{L}}{d\sigma} &= \sum_{k=1}^N \frac{2(z_k-\gamma)^2}{\sigma(\sigma^2+(z_k-\gamma)^2)} - \frac{N}{\sigma}= 0.
    \end{split}
\end{equation}
We can see that solving $\gamma$ requires handling a polynomial of degree $2n-1$, and solving $\sigma$ requires computing a polynomial of degree $2n$. This tends to be complicated by the fact that this requires finding the roots of a high degree polynomial, and there can be multiple roots that represent local maximum \cite{Ferguson1978MaximumLE}. In \cite{Zhangwenzhong1991counterexamplesinprobabilitystatistics}, the authors give a counterexample that the likelihood equation has multiple roots and the maximum likelihood estimator may not be unique. They also show that the maximum likelihood estimator can only be approximated by numerical methods which makes it very inconvenient to apply the usual maximum likelihood estimator method to Cauchy distribution. Meanwhile, because the expectation and variance of Cauchy distribution don't exist, the usual moment estimation method is not suitable for the parameter estimation of Cauchy distribution.

\emph{Then}, since the general $\alpha$-stable distribution doesn't have a closed-form of its probability density function (Appendix \ref{alpha rv}), it's natural to wonder the rationality of the maximum log-likelihood function as a loss function. We follow the initial loss function and construction in Section \ref{challenge}, that is, only the negative log-likelihood function is used as the loss function. An experiment of SDE (\ref{experiment of only mle}) with training results of neural networks is shown in the Figure \ref{fig:only mle loss}.

Unsurprisingly, as discussed in the Cauchy noise case, the results of simply using the log-likelihood function as a loss are poor. We give the following explanations: the probability density functions of $\alpha$-stable distributions are too complex, the maximum likelihood estimation falls into local extremes; estimating two unknown quantities with a function is inherently uncertain. On the other hand, we cannot completely negate the log-likelihood function as a loss function. Looking at Figure \ref{fig:only mle loss}, it can be found that the estimated results of the drift coefficient $f_\theta$ and the diffusion coefficient $g_\theta$ are better and worse at the same time. Intuitively, the maximum likelihood estimator needs an auxiliary method to determine its value to prevent the optimization process of the maximum likelihood estimator from falling into local minimums. 

To sum up, the log-likelihood function is not a panacea. Constructing the loss function solely by maximizing likelihood function is sufficient for the Brownian noise case and small noise, but not for $\alpha$-stable L\'evy noise case. We will discuss how to build neural networks and loss functions in different stability parameter scenarios.

\subsection{Two-step hybrid method (1-D)} \label{proposed method}
In previous section, we see problems with log-likelihood functions for $\alpha$-stable L\'evy case, hence here we explore the solutions through a two-step learning framework.

\subsubsection{A special case --- Cauchy} \label{cauchy structure}
Let us begin with the special case of identifying stochastic dynamical systems with Cauchy noise ($\alpha=1$).

Inspired by \cite{TheparameterestimateofWangbingzhang2021Cauchydistribution}, the least absolutely deviation(LAD) estimation is advocated by Laplace which has the advantage of robustness. The introduction of the robustness requirement is one of the major advances in statistics in the second half of the 20th century. By definition, the least absolutely deviation estimate is the value $\gamma$ that minimizes the mean absolute deviation:
\begin{equation}
  H(\gamma)=\sum_{k=1}^N|z_k-\gamma|,  \label{LAD}
\end{equation}
where $(z_1,z_2,...,z_N)$ represent the $N$ i.i.d. samples of Cauchy distribution with parameters $\gamma$ and $\sigma$. If we find the least absolutely deviation of the parameter $\gamma$, denoted as $\hat{\gamma}_N$, the parameter $\sigma$ can be estimated as follows:
 \begin{equation} 
     \hat{\sigma}_N=\frac{1}{2}\Big[\frac{1}{N}\sum_{k=1}^N\sqrt{|z_k-\hat{\gamma}_N|}\Big]^2. \label{Cauchy sigma}
 \end{equation}
Note that $\hat{\sigma}_N$ is calculated directly by the formula (\ref{Cauchy sigma}). Moreover, the parameters $\hat{\gamma}_N$ and $\hat{\sigma}_N$ obtained by the above method are strong consistency estimator, i.e., $\hat{\gamma}_N$ and $\hat{\sigma}_N$ converge to $\gamma$ and $\sigma$ with probability 1. Thus, we propose our learning schemes in the following two steps.

\textbf{Step 1:} Corresponding to our task, for the drift coefficient $f_\theta$, we construct a loss function similar to (\ref{LAD}): 
\begin{equation} \label{cauchy loss1}
    \begin{split}
        \mathcal{L}(\theta|x_1,x_0,h) = \sum_{k=1}^N|x_1^{(k)}-(x_0^{(k)} + hf_\theta(x_0^{(k)}))|.
    \end{split}
\end{equation}
By minimizing the loss function (\ref{cauchy loss1}), we obtained the drift item $\hat{f}_\theta$ we expected. 

\textbf{Step 2: }As for the diffusion coefficient $g_\theta$, it can be obtained by direct calculation of the analog of formula (\ref{Cauchy sigma}):
\begin{equation} \label{cauchy loss2}
    \hat{g}_\theta(x_0^{(i)})=\frac{1}{2}\Big[\frac{1}{N_i}\sum_{k_i=1}^{N_i}\sqrt{|x_1^{(k_i)}-(x_0^{(i)} + h\hat{f}_\theta(x_0^{(i)}))|}\Big]^2, \ for\  i = 1, 2, \dots, n,
\end{equation}
where $N_1+N_2+\dots+N_n = N$.

\subsubsection{General cases when $\alpha \in (0,2) \setminus \{1\}$} \label{levy structure}

In this section, we will find that the log-likelihood function plays a key role in training. Since it contains the probability density function, we will first introduce the probability density function of the $\alpha$-stable random variable we used.

Although there are three probability density functions available (Appendix \ref{alpha pdf cf}), as they all contain integrals or infinite series, it is difficult to maximize the log-likelihood function. After programming experiments, Zolotarev formula is the most stable and accurate one. However, because of the truncation error from approximated summation, other two functions cannot give us appropriate approximations when state space is wide and density is comparatively smooth.

\textbf{Zolotarev formula}:
The probability density function of $X \sim S_\alpha(1,0,0)$ is:
\begin{equation}  \label{pdf3 in main body}
    p(x;\alpha, \beta=0, \sigma=1, \gamma=0)=
    \left\{
    \begin{array}{lr}
    \dfrac{\alpha x^{\frac{1}{\alpha-1}}}{\pi|\alpha-1|}\int_0^{\frac{\pi}{2}}V(\theta;\alpha,\beta=0)\exp\{-x^{\frac{1}{\alpha-1}}V(\theta;\alpha,\beta=0)\}d\theta,\  for\ x > 0,\\
    \\
    \dfrac{\Gamma(1+\frac{1}{\alpha})}{\pi},\  for\ x = 0,\\
    \\
    p(-x;\alpha, -\beta=0, \sigma=1, \gamma=0),\  for\ x < 0,\\
    \end{array}
    \right.
\end{equation}
where $V(\theta;\alpha,\beta=0)=\Big(\dfrac{\cos\theta}{\sin\alpha\theta}\Big)^{\frac{\alpha}{\alpha-1}}\cdot\dfrac{\cos\{(\alpha-1)\theta\}}{\cos\theta}$. A description of other forms of probability density functions is in the Appendix \ref{alpha pdf cf}. Unless otherwise specified, the $\alpha$-stable probability density function we generally use is the Zolotarev formula (\ref{pdf3 in main body}, \ref{pdf3 }). We introduce a proposition that has important implications for subsequent derivations:

\begin{proposition}
(i) If $X \sim S_\alpha(\sigma,\beta,\gamma)$ and $a$ is a real constant, then $X + a \sim S_\alpha(\sigma,\beta,\gamma + a)$.\\
(ii) If $X \sim S_\alpha(\sigma,\beta,\gamma)$ and $k$ is a real constant, then 
\begin{equation}
    kX \sim 
    \left\{
    \begin{array}{lr}
    S_\alpha(|k|\sigma,\sign k\beta,k\gamma),\  \alpha \neq 1,\\
    \\
     S_1(|k|\sigma,\sign k\beta,k\gamma-\frac{2}{\pi}k(\log|k|)\sigma\beta),\  \alpha=1.\\
    \end{array}
    \right.
\end{equation}

In particular, if $X \sim S_\alpha(1,0,0)$ and $k$ is a real constant, then $kX \sim S_\alpha(|k|,0,0)$, for $\alpha \in (0,2)$.\\
(iii) If $X \sim S_\alpha(\sigma,\beta,0)$, then $-X \sim S_\alpha(\sigma,-\beta,0)$, for $\alpha \in (0,2)$.\\
(iv) If $X_1$, $X_2$ are independent stable random variables with the same distribution $S_\alpha(\sigma,\beta,\gamma)$ for $\alpha \neq 1$ and $A$, $B$ are positive constants, then
\begin{equation*}
    AX_1+BX_2 \sim S_\alpha(\sigma(A^\alpha+B^\alpha)^{\frac{1}{\alpha}},\beta,\gamma(A+B)).
\end{equation*}
\label{alpha-stable rv properties}
\end{proposition}

\begin{proof}
See \cite{Samorodnitsky1995StableNR}, Chapter 1.
\end{proof}

Hence, if $p_{\alpha}(x;\sigma,\beta,\gamma)$ is the probability density function of the stable random variable $X \sim S_\alpha(\sigma,\beta,\gamma)$, then $p_{\alpha}(x;\sigma,\beta,\gamma+a)$ is the probability density function of $X+a$ (for every real constant $a$) and $p_{\alpha}(x;\sigma A,\beta,\gamma A)$ is the probability density function of $AX$ (for every positive constant $A$ and $\alpha \ne 1$).

Similar to the Gaussian noise, it also starts from Euler-Maruyama discretization method (\ref{E-M LDE}) for training the networks $f_\theta$ and $g_\theta$. Essentially, conditioned on $x_0$ and $h$, we can think of $x_1$ as a point extracted from a $\alpha$-stable distribution:
\begin{equation}
    x_1 \sim  S_\alpha(g(x_0)h^{1/\alpha},0,x_0+hf(x_0)), \label{x1 distribution}
\end{equation}
where we use the argument (i) and (ii) in Proposition \ref{alpha-stable rv properties} and the fact $L^{\alpha}_h \sim S_\alpha(h^{\frac{1}{\alpha}},0,0)$. As mentioned in Zolotarev fomula (\ref{pdf3 in main body}, \ref{pdf3 }), we introduce the probability density function of $X \sim S_\alpha(1, 0 , 0),\ \alpha \in (0,2) \setminus \{1\}$. In order to get the probability density function of (\ref{x1 distribution}), we need a simple derivation.

\begin{remark} \label{remark}
If $X \sim S_\alpha(1,0,0)$ has probability density function $p_\alpha(x)$, $b$, $A>0$ are constants, then $b+AX \sim S_\alpha(A,0,b)$ has the probability density function $\frac{1}{A} p_\alpha(\frac{x-b}{A})$. This comes from the fact that $\mathbb{P}(b+AX \le x) = \mathbb{P}(X \le \frac{x-b}{A}) = \int_{-\infty}^{\frac{x-b}{A}}p_\alpha(\xi)d\xi$ and $\frac{d}{dx}\mathbb{P}(b+AX \le x) = \frac{1}{A}p_\alpha(\frac{x-b}{A})$.
\end{remark}

After preparing the required probability density function and its transformation formula, we first need to construct an auxiliary measure for the log-likelihood function. In our method, it manifests as the first step in the two-step estimation method, i.e., estimating drift coefficient $f$.

\textbf{Step 1: }Observation of the $x_1$ distribution (\ref{x1 distribution}) shows that the drift coefficient $f$ is essentially related to the location parameter of the $\alpha$-stable random variable. Taking inspiration from improvements to Cauchy estimation, we choose to first estimate the drift $f$ or the location parameter of $x_1$ distribution using a neural network. From a machine learning perspective, the estimation of the Cauchy drift coefficient $f_\theta$ can essentially be considered as Mean Absolute Error (MAE). After testing, we find that the Mean square Loss (MSE) converges faster than MAE and the results are better. For the drift neural network $f_\theta$, we use the following loss function:
\begin{equation}
    \mathcal{L}(\theta|x_1,x_0,h) = \frac{1}{N}\sum_{k=1}^N[x_1^{(k)}-(x_0^{(k)} + hf_\theta(x_0^{(k)}))]^2. \label{general drift loss}
\end{equation}
Minimizing the above loss function (\ref{general drift loss}), we obtain the desirable drift estimation $\hat{f}_\theta$. 

\textbf{Step 2: }Since the estimation of the drift $\hat{f}_\theta$ assists the log-likelihood function, we can now formulate the log-likelihood loss function that will be minimized to obtain the neural network weights $\theta$ for diffusion network $g_\theta$. The log-likelihood formula of the data set $D$ is no change in form (\ref{theta def}) except the notion $p_\theta$ means the probability density function of (\ref{x1 distribution}). If we bring in the probability density function (\ref{pdf3 in main body}) and apply the conclusion from Remark \ref{remark}, we can get a loss function of the analogous form (\ref{mle gaussian loss}). We note that the specific form (\ref{x1 gaussian distribution}) and (\ref{x1 distribution}) of the transition probabilities $p_\theta$ induced by the Euler-Maruyama method implies that it is only the conditional probability given $x_0$ and $h$. So we need data pairs such as $(x_0, x_1)$, and the conclusion that the conditional distribution of $x_1$ after multi-step evolution is still a $\alpha$-stable distribution is often wrong. 

The logarithm of the probability density function of the $\alpha$-stable distribution, together with four parameters from (\ref{x1 distribution}), yields the loss function to minimize, that is,
\begin{equation} \label{mle loss}
\begin{split}
    \mathcal{L}(\theta|x_1,x_0,h):&= -\frac{1}{N}\sum_{k=1}^{N}\log p_{\theta}(x_1^{(k)}|x_0^{(k)},h)\\
    &= -\frac{1}{N}\sum_{k=1}^{N}\log \Big[ \frac{1}{g_\theta(x_0^{(k)})h^{1/\alpha}}p_\alpha \Big(\frac{x_1^{(k)}-(x_0^{(k)}+hf_\theta(x_0^{(k)}))}{g_\theta(x_0^{(k)})h^{1/\alpha}}\Big)\Big]\\
    &= \frac{1}{N}\sum_{k=1}^{N}\log(g_\theta(x_0^{(k)})h^{1/\alpha}) - \frac{1}{N}\sum_{k=1}^{N} \log \Big[p_\alpha \Big(\frac{x_1^{(k)}-(x_0^{(k)}+hf_\theta(x_0^{(k)}))}{g_\theta(x_0^{(k)})h^{1/\alpha}}\Big)\Big]
\end{split}
\end{equation}
where $p_\alpha$ represents the probability density function of the $\alpha$-stable random variable $S_\alpha(1,0,0)$. Minimizing $\mathcal{L}$ in (\ref{mle loss}) over the data $D$ is equivalent to maximization of the log-likelihood function (\ref{theta def}).

Using the estimated drift coefficient $\hat{f}_\theta$, bring in to (\ref{mle loss}) and minimize, we can get the desired diffusion estimation $\hat{g}_\theta$. Such a two-step training method avoids weight adjustment between loss functions on the one hand and gives the log-likelihood function an auxiliary estimation on the other hand.

\subsection{Multi-dimensional case} \label{multi_dim}

We continue with the setting in Section \ref{sde setting}, i.e., diffusion coefficient $g(X_t)$ is a $d \times d$ diagonal matrix and the $d$-dimensional $\alpha$-stable L\'evy motion is composed of $d$ mutually independent one-dimensional symmetric $\alpha$-stable L\'evy motions. Based on these assumptions, we can generalize our algorithm to multi-dimensional stochastic differential equations.

\textbf{Step 1:} Compared to Step 1 of Section \ref{levy structure}, we rewrite the loss function (\ref{general drift loss}) to the average of $d$ dimensions:
\begin{equation}
    \mathcal{L}(\theta|x_1,x_0,h) = \frac{1}{Nd} \sum_{k=1}^N \sum_{i=1}^d[x_{1,i}^{(k)}-(x_{0,i}^{(k)} + hf_{\theta,i}(x_{0}^{(k)}))]^2. \label{d-dim drift loss}
\end{equation}
where $x_j=(x_{j,1},x_{j,2},\dots,x_{j,d})^T$ for $j=0, 1$ and $f_\theta(x_0)=(f_{\theta,1}(x_0),f_{\theta,2}(x_0),\dots,f_{\theta,d}(x_0))^T$ are $d$-dimensional vectors, time step $h$ is a scalar value.

\textbf{Step 2:} After getting the estimated drift coefficient $\hat{f}_\theta$, the diffusion coefficient $g$ is the only thing left to figure out. Since $g$ is a diagonal matrix and the components of L\'evy motion are mutually independent, under conditional probability, we can write the joint distribution as the product of marginal distributions:
\begin{equation} \label{d-dim mutually independent levy pdf}
    p(x_1|x_0,h,f(x_0))  = p(x_{1,1}|x_0,h,f(x_0)) p(x_{1,2}|x_0,h,f(x_0)) \dots p(x_{1,d}|x_0,h,f(x_0)),
\end{equation}
where $x_1$, $x_0$ and $f(x_0)$ are $d$-dimensional vectors, time step $h$ is a scalar value. Take the logarithm of the equation (\ref{d-dim mutually independent levy pdf}):
\begin{equation} \label{log d-dim mutually independent levy pdf}
    \log p(x_1|x_0,h,f(x_0))  = \sum_{i=1}^d \log p(x_{1,i}|x_0,h,f(x_0)).
\end{equation}
Furthermore, we can briefly rewrite the formula (\ref{mle loss}): 
\begin{equation} \label{d-dim mle loss}
    \mathcal{L}(\theta|x_1,x_0,h,\hat{f}_\theta(x_0))  
    = -\frac{1}{N}\sum_{k=1}^{N}\sum_{i=1}^d \log p_{\theta}(x_{1,i}^{(k)}|x_0^{(k)},h,\hat{f}_\theta(x_0^{(k)})).
\end{equation}
The required diffusion coefficient $g$ is estimated by $\hat{g}_\theta$ through minimizing the formula (\ref{d-dim mle loss}).

Note that this generalization of $d$-dimensional SDEs is because the $d$-dimensional probability density function that can be written as the product of one-dimensional marginal density functions (\ref{d-dim mutually independent levy pdf}). Further, we can explain assumptions that $g(X_t)$ is a diagonal matrix and that the components of L\'evy motion are mutually independent in Section \ref{sde setting}. Without the assumption that $g(X_t)$ is a diagonal matrix, we consider a two-dimensional SDE and use the Euler-Maruyama scheme (\ref{E-M LDE}) represented as follows:
\begin{equation} \label{2d explain}
    \begin{aligned}
        X_1&= X_{0} + h f_1(X_{0},Y_{0}) + A_1 L_{h,1}^{\alpha} + A_2 L_{h,2}^{\alpha},\\
        Y_1&= Y_{0} + h f_2(X_{0},Y_{0}) + B_1 L_{h,1}^{\alpha} + B_2 L_{h,2}^{\alpha},
    \end{aligned}
\end{equation}
where constants $A_1,A_2,B_1,B_2 >0$ and $L_{h,1}^{\alpha},L_{h,2}^{\alpha}$ are independent and $\alpha$-stable distributed. Although the linear combination of independent $\alpha$-stable random variables is still an $\alpha$-stable random variable (Proposition \ref{alpha-stable rv properties}, (iv)), $X_1$ and $Y_1$ aren't independent, and the joint distribution of the random vector $(X_1, Y_1)^T$ cannot be expressed as the product of marginal distributions. Without the assumption that the components of L\'evy motion are mutually independent, the explicit probability density function of this two-dimensional random vector $(X_1, Y_1)^T$ is unknown. Besides that, we introduce another case of $d$-dimensional $\alpha$-stable random vectors. In \cite{Nolan2013MultivariateEC}, the author discussed the $d$-dimensional elliptically contoured $\alpha$-stable distribution whose probability density function can be computed through amplitude distribution and surface integral. However, its complicated probability density function leads to intractable approximations, which is beyond the scope of our article.

\section{Experiments} \label{experiments}

We illustrate the neural network identification techniques in the following examples, divided by different ranges of $\alpha$ value ($\alpha = 1$, $\alpha \in (1,2)$, $\alpha \in (0,1)$), the categories of noise (additive noise or multiplicative noise) and the dimensions of SDEs. We created snapshots by integrating the SDEs (\ref{LDE}) from $t = 0$ to $t = h$. To be specific, these samples come from simulating stochastic differential equation with initial value $x_0$ using Euler-Maruyama scheme, where $x_0 \in [-3,3]$, $x_0 \in [-1,1]$, $x_0 \in [-2.5,2.5]$ for different examples. 

In this article, we consider stochastic dynamical systems induced by medium noise or large noise, and noise very close to 0 is not considered (noise very close to 0 can be approximated as a deterministic system). Furthermore, we use a Softplus activation function before output $g_\theta$, to guarantee $g_\theta$ is always positive. In addition, we clipped the lower bound of output $g_\theta$ with $10^{-13}$, so that we won’t have loss go to infinity.

To avoid confusion, we post a list of experiments in Table \ref{tab:exper}. Since the drift coefficients and diffusion coefficients learned by the neural network are implicitly expressed, and the negative log-likelihood function is not necessarily greater than $0$ as the general loss function, we calculate the $\mathcal{L}^2$ error of the two coefficients in addition to the figures to explicitly quantify the experimental results: 
\begin{equation}
    \begin{aligned}
        \mathcal{L}_f^2 = \frac{1}{N} \sum_{k=1}^{N}(f(x_0^{(k)})-\hat{f}_\theta(x_0^{(k)}))^2,\ 
        \mathcal{L}_g^2 = \frac{1}{N} \sum_{k=1}^{N}(g(x_0^{(k)})-\hat{g}_\theta(x_0^{(k)}))^2.
    \end{aligned}
\end{equation}
The results can be seen in the last two columns of Table \ref{tab:exper}.

\begin{table}[H]
  \centering
  \caption{List of experiments (Experiments with font bolding can be found in this section and the remaining experiments are shown in Appendix \ref{other experiments})}
  \begin{tabular}{c|l|c|c|c|c}
    \hline
    Stability & \multirow{2}{*}{Noise type} & Drift & Diffusion & Error & Error\\
    parameter &  & coefficient $f$ & coefficient $g$ & $\mathcal{L}^2_f$ & $\mathcal{L}^2_g$\\
    \hline\hline
    
    \multirow{6}{*}{$\alpha=1$}& \multirow{3}{*}{additive} & $-X_t+1$  & $1$ & 0.0021 & 0.0004\\
    \cline{3-6}
    & &$\bm{-X_t^2}$ & $\bm{1}$ & \textbf{0.0020} & \textbf{0.0000} \\
    \cline{3-6}
    & &$\sin(X_t)$& $1$ & 0.0080 & 0.0017 \\
    \cline{2-6}
    & \multirow{3}{*}{multiplicative} & $-X_t+1$ & $0.1X_t+0.5$ & 0.0041 & 0.0147 \\
    \cline{3-6}
    & &$\bm{-X_t^2}$ & $\bm{0.1X_t+0.5}$ & \textbf{0.0344} & \textbf{0.0136}\\
    \cline{3-6}
    & &$\sin(X_t)$ & $0.1X_t+0.5$ & 0.0054 & 0.0173\\
    
    \hline
    \multirow{6}{*}{$\alpha=1.5$}& \multirow{4}{*}{additive} & $-X_t+1$ & $1$& 0.0038 & 0.0007\\
    \cline{3-6}
    & &$-X_t^3+X_t$ & $1$& 0.0010 & 0.0003\\
    \cline{3-6}
    & &$\log (X_t+1.5)$ & \multirow{2}{*}{$1$} & \multirow{2}{*}{0.0003} & \multirow{2}{*}{0.0008}\\
    & &$-|X_t+1.5|^{1/3}$ & &\\
    \cline{2-6}
    & \multirow{2}{*}{multiplicative} & $\bm{-X_t^3+X_t}$ & $\bm{X_t+1}$ & \textbf{0.0009} & \textbf{0.0007} \\
    \cline{3-6}
    & &$\bm{-X_t^3+X_t}$ & $\bm{\sin(\pi X_t)+1}$ & \textbf{0.0006} & \textbf{0.0108}\\
    
    \hline
    \multirow{4}{*}{$\alpha=0.5$}& \multirow{2}{*}{additive} & $-X_t+1$ & $1$& 0.0001 & 0.0045\\
    \cline{3-6}
    & &$-X_t^3+X_t$ & $1$& 0.0002 & 0.0010\\
    \cline{2-6}
    & \multirow{2}{*}{multiplicative} & $-X_t+1$ & $X_t+1$ &0.0001 & 0.0139\\
    \cline{3-6}
    & &$-X_t^3+X_t$ & $X_t+1$& 0.0002 & 0.0047\\
    \hline
    compared with nonlocal & \multirow{2}{*}{additive} & \multirow{2}{*}{$\bm{-X_t^3+4X_t}$} & \multirow{2}{*}{$\bm{1}$} &\multirow{2}{*}{\textbf{0.0014}} & \multirow{2}{*}{\textbf{0.0003}}\\ Kramers-Moyal & & & & &\\
    \hline
    \multirow{4}{*}{(2-D) $\alpha=1.5$}& \multirow{2}{*}{additive} & $\bm{X_t-X_t^3-X_t Y_t^2}$  & $\bm{1}$ & \multirow{2}{*}{\textbf{0.0014}}& \multirow{2}{*}{\textbf{0.0026}}\\
    \cline{3-4}
    & & $\bm{-(1+X_t^2)Y_t}$ & $\bm{1}$ & &\\
    \cline{2-6}
    & \multirow{2}{*}{multiplicative} & $X_t+Y_t$ & $0.5Y_t+1$ & \multirow{2}{*}{0.0024} & \multirow{2}{*}{0.0020}\\
    \cline{3-4}
    & & $4X_t-2Y_t$& $0.5X_t+1$ & &\\
    \hline
  \end{tabular}
  
  \label{tab:exper}
\end{table}

\subsection{SDEs driven by Cauchy motion} \label{cauchy experiments}

In this section, we test the special case of $\alpha=1$. The drift neural networks $f_\theta$ in our tests have two hidden layers and 25 neurons per layer, batch size of 100, 100 epochs, with exponential linear unit (ELU) activation function and the ADAMAX optimizer with default parameters. The diffusion coefficients $g_\theta$ are calculated by formula (\ref{cauchy loss2}). 

We present two representative examples in this section: (1) additive noise, i.e., $g(X_t)$ in (\ref{LDE}) is a constant, we assume $g(X_t) = 1$ without loss of generality; (2) multiplicative noise with a diffusion coefficient of $g(X_t) = 0.1X_t+0.5$. Both drift coefficients are $f(X_t)=-X_t^2$. Here we take sample size $N=10000$, $h = 0.01$, $x_0 \in [-3,3]$ for example (1). Example (2) follows the construction of example (1) except for sample size $N = 20$ (different $x_0$) $\times1000$ (sample size). By comparing the real coefficients with the learned coefficients in Figure \ref{fig:cauchy-add-mul}, we can see that they overlap very well.

\begin{figure}[H]
  \centering
  \subfigure[$f(X_t)=-X_t^2$, $\hat{g}_{\theta} = 0.9955 $]{
    \label{fig:subfig:cauchy-squa-add} 
    \includegraphics[scale=0.65]{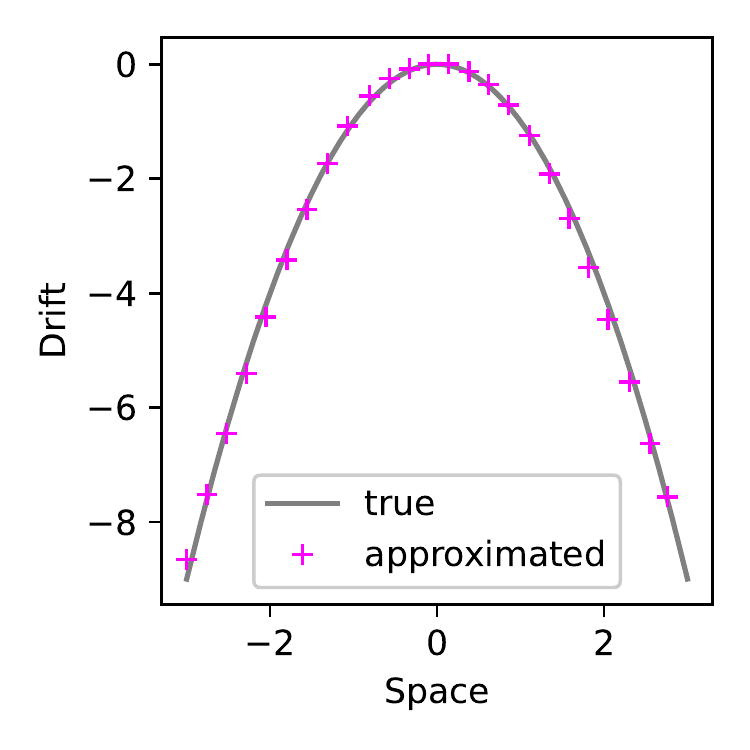}}
  \hspace{0.1in} 
  \subfigure[$f(X_t)=-X_t^2$, $g(X_t) = 0.1X_t+0.5$, (left) true and approximated drift coefficients, (right) true and approximated diffusivity coefficients]{
    \label{fig:subfig:cauchy-squa-mul} 
    \includegraphics[scale=0.65]{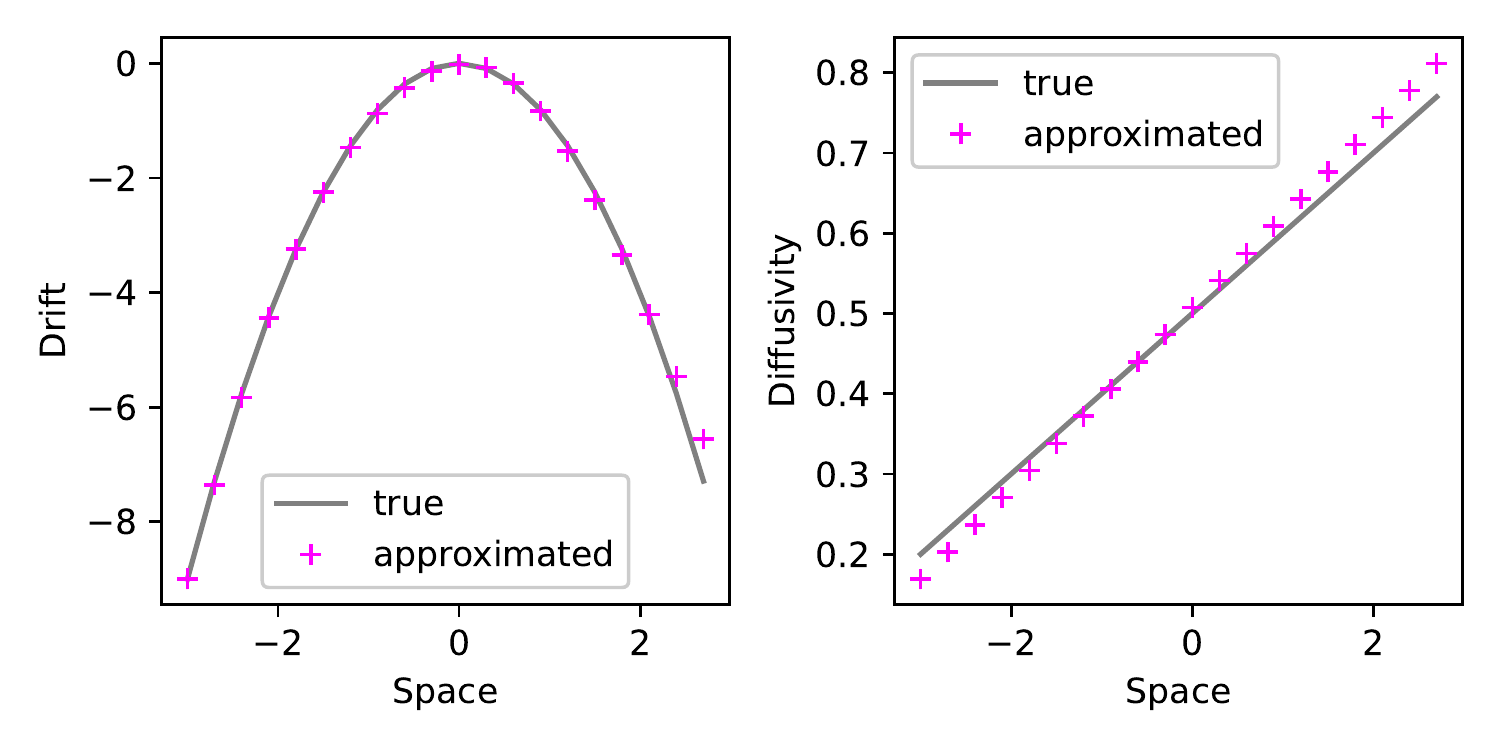}}
  \caption{Cauchy noise with two representative examples, the gray lines represent the true coefficients and the magenta plus markers are the neural networks' results. $\hat{g}_{\theta}$ represents the diffusion coefficient calculated by (\ref{cauchy loss2}).}
  \label{fig:cauchy-add-mul} 
\end{figure}

\subsection{SDEs driven by $\alpha$-stable L\'evy motion} \label{one dim levy experiments}

In this section, neural networks are used twice, separately to estimate drift and diffusion. The drift neural networks $f_\theta$ have three hidden layers and 25 neurons per layer, no batch normalization, 300 epochs, with exponential linear unit (ELU) activation function, and the ADAM optimizer with a learning rate of $0.005$ and other default parameters. The diffusion neural networks $g_\theta$ have two hidden layers and 25 neurons per layer, output layer with the Softplus function to ensure that the diffusion coefficients are positive, batch size of 512, no more than $30$ epochs, with exponential linear unit (ELU) activation function and the ADAMAX optimizer with a learning rate of $0.005$, eps $=10^{-7}$ and other default parameters.

We separately discuss $\alpha \in (1,2)$ and $\alpha \in (0,1)$. This is due to the fact that as $\alpha$ decreases, the probability distribution function presents a greater probability near the location parameter and a greater probability of the tail, which results in a difference in the accuracy of the estimated result.  Without loss of generality, we select $\alpha=1.5$ and $\alpha=0.5$. 

As described in Table \ref{tab:exper} and Appendix \ref{other experiments}, our experiments on drift $f$ focus on linear drift coefficients, double-well drift coefficients, and complex drift coefficients. Additive noise remains $g=1$. Multiplicative noise contains linear diffusion coefficients and trigonometric diffusion coefficients. We use a preprocessing trick to speed up training and improve estimation accuracy, see Appendix \ref{drift trick}. Different experimental parameters may vary slightly. We give an appropriate explanation for different step sizes. Indeed, take a glance at the loss function (\ref{general drift loss}), our goal is to find the optimal $f_\theta$ but it is affected by the step size, and too small a step size can make it difficult to estimate $f$. On the other hand, too large a time step can cause the precision of the Euler-Maruyama scheme to decrease. Therefore, there is a trade-off between them. For a smaller time step, bringing more samples into the training process is one way to achieve the same estimated precision.

We look at representative examples once more and provide suggestions for improvement. For $\mathbf{\alpha=1.5}$, we test the drift coefficients $f(X_t) = -X_t^3+X_t$, diffusion coefficient $g(X_t) = X_t+1$ and $g(X_t) = \sin (\pi X_t) + 1$ in equation (\ref{LDE}). Here we take sample size $N=50\times1000$, $h = 0.5$, $x_0 \in [-1,1]$. We compare the true coefficients and the learned coefficients in Figure \ref{fig:1-2-multi}.

\begin{figure}[H]
  \centering
  \subfigure[$f(X_t)=-X_t^3+X_t$, $g(X_t) = X_t+1$, (left) true and approximated drift coefficients, (right) true and approximated diffusivity coefficients]{
    \label{fig:subfig:1-2-dw-multi1} 
    \includegraphics[scale=0.48]{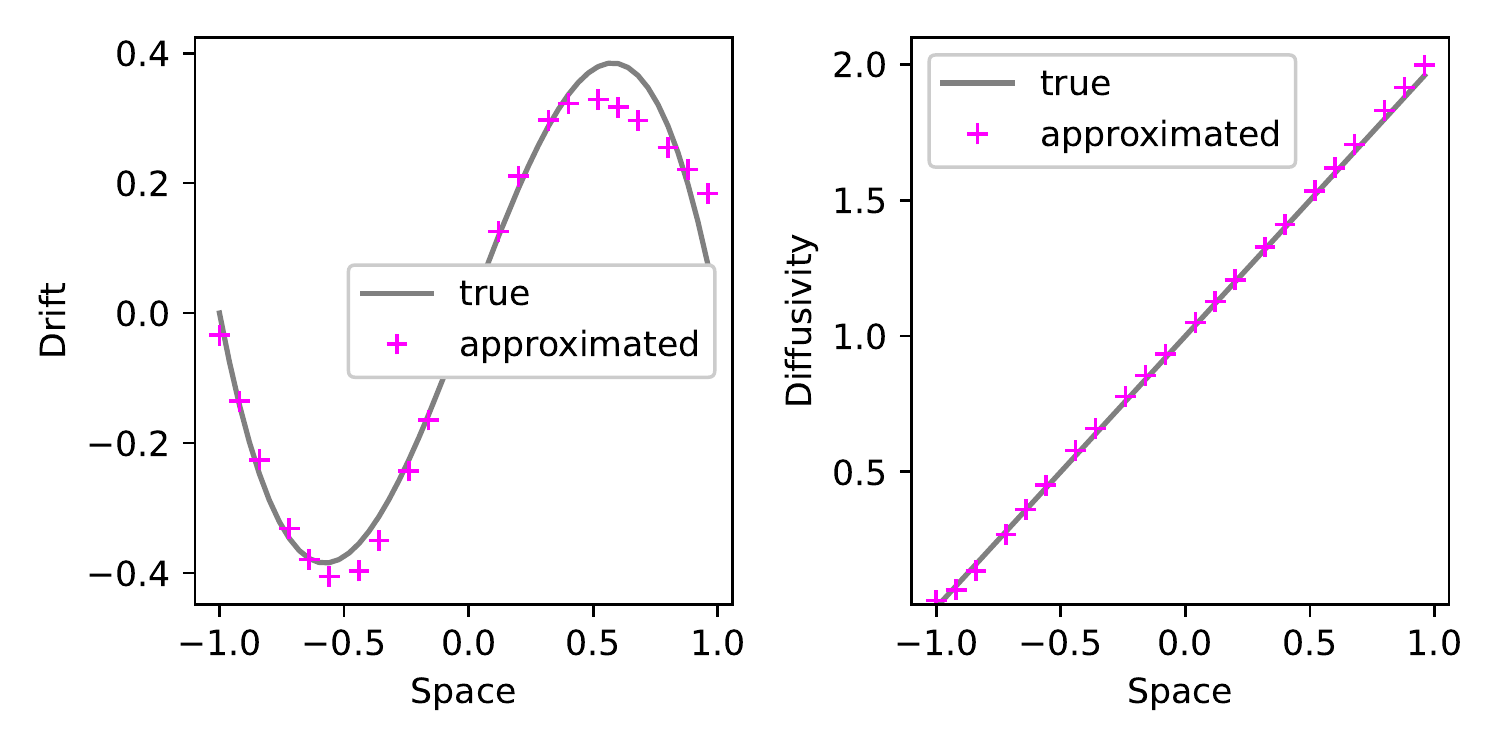}}
  \hspace{0.1in} 
  \subfigure[$f(X_t)=-X_t^3+X_t$, $g(X_t) = \sin(\pi X_t) + 1$, (left) true and approximated drift coefficients, (right) true and approximated diffusivity coefficients]{
    \label{fig:subfig:1-2-dw-multi2} 
    \includegraphics[scale=0.48]{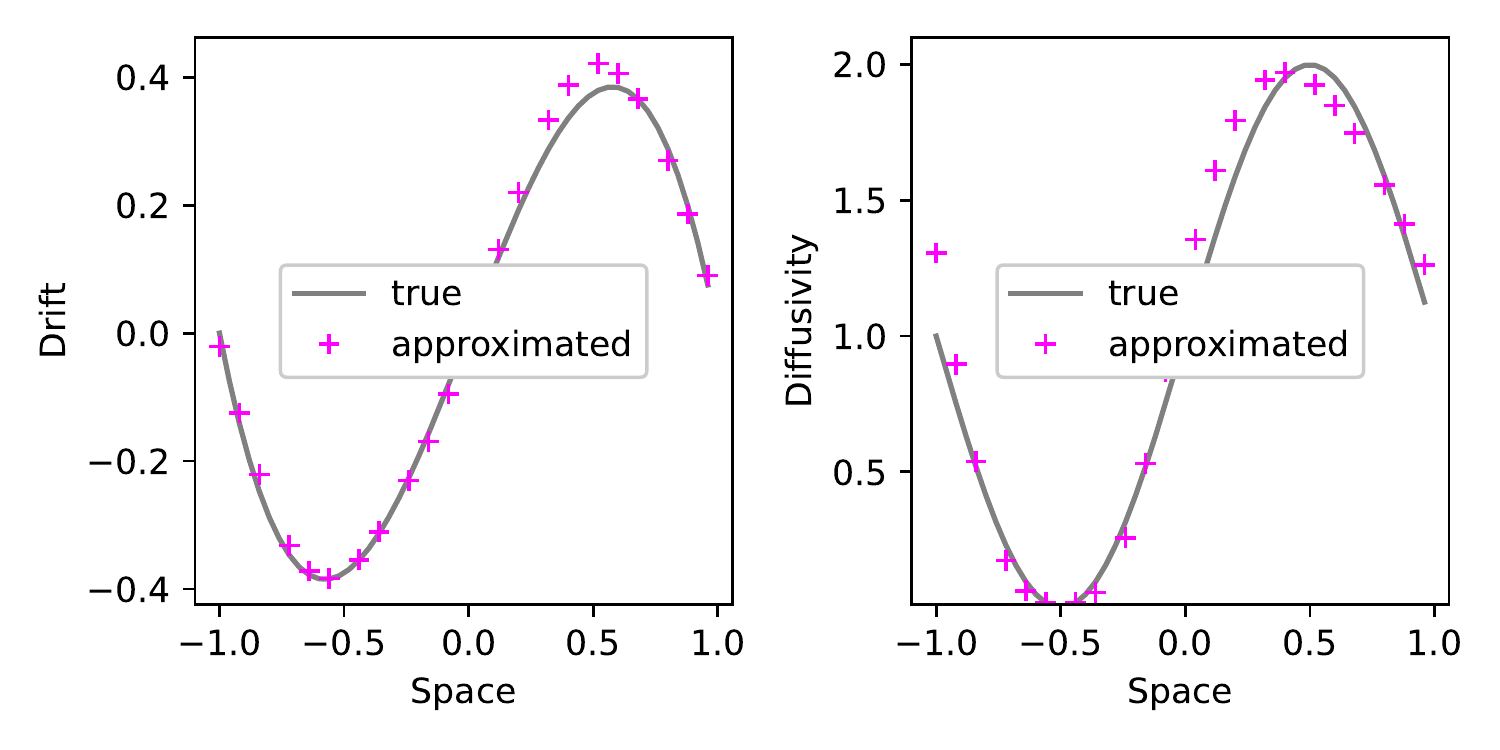}}
  \caption{Multiplicative noise with different diffusion coefficients for $\alpha=1.5$, the gray lines represent the true coefficients and the magenta plus markers are the estimated drift and diffusion results from our proposed neural networks.}
  \label{fig:1-2-multi} 
\end{figure}

~\\
\textbf{Improvement of accuracy}:

A cursory look reveals that the estimation results are poor when the noise is large, i.e., the poor performance at the right end of the axis of Figure \ref{fig:1-2-multi}. We add more samples at the large noise, and the data volume becomes $75\times1000$ which the extra samples are added to the interval $(0,1)$. The improved results are shown in Figure \ref{fig:1-2-multi-more}.

\begin{figure}[H]
  \centering
  \subfigure[$f(X_t)=-X_t^3+X_t$, $g(X_t) = X_t+1$, (left) true and approximated drift coefficients, (right) true and approximated diffusivity coefficients]{
    \label{fig:subfig:1-2-dw-multi1-more} 
    \includegraphics[scale=0.48]{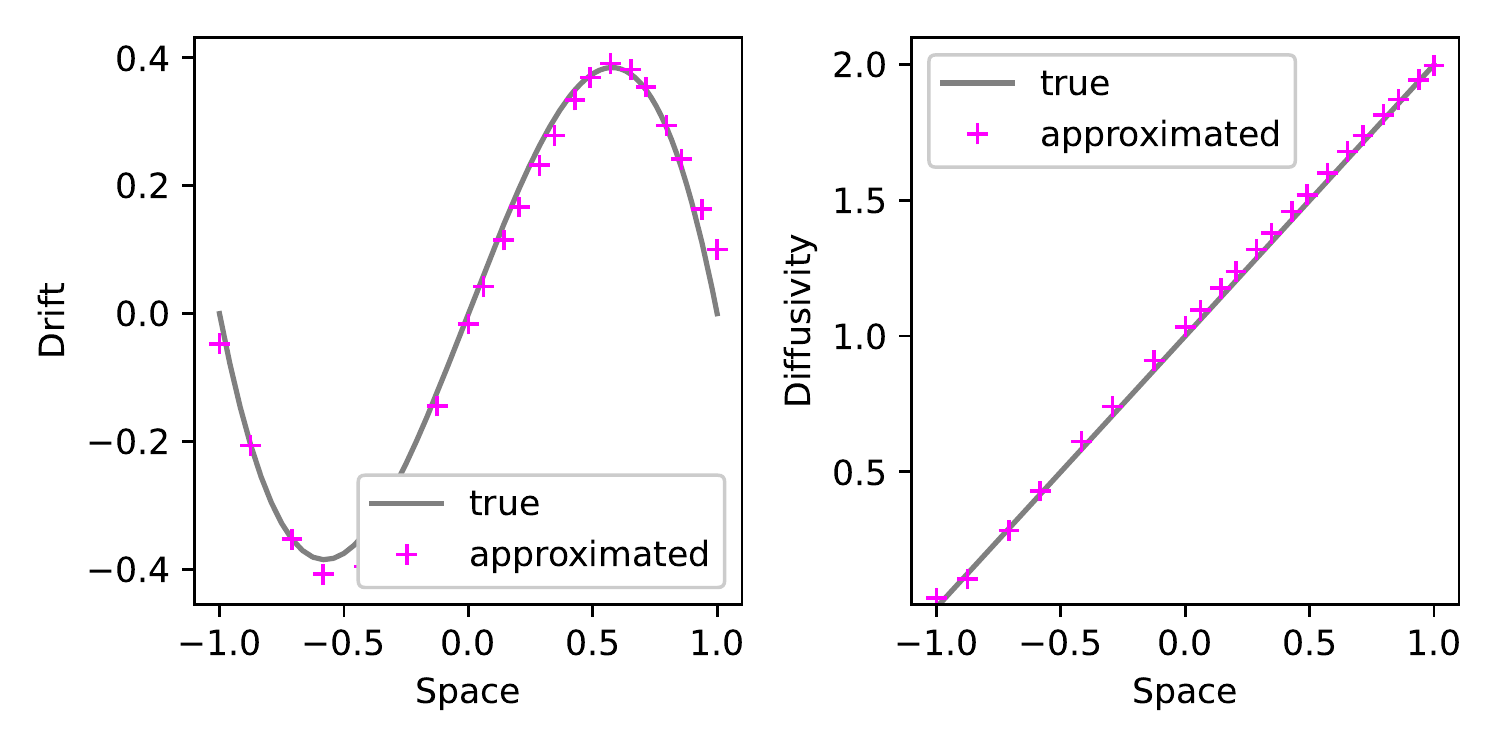}}
  \hspace{0.1in} 
  \subfigure[$f(X_t)=-X_t^3+X_t$, $g(X_t) = \sin(\pi X_t) + 1$, (left) true and approximated drift coefficients, (right) true and approximated diffusivity coefficients]{
    \label{fig:subfig:1-2-dw-multi2-more} 
    \includegraphics[scale=0.48]{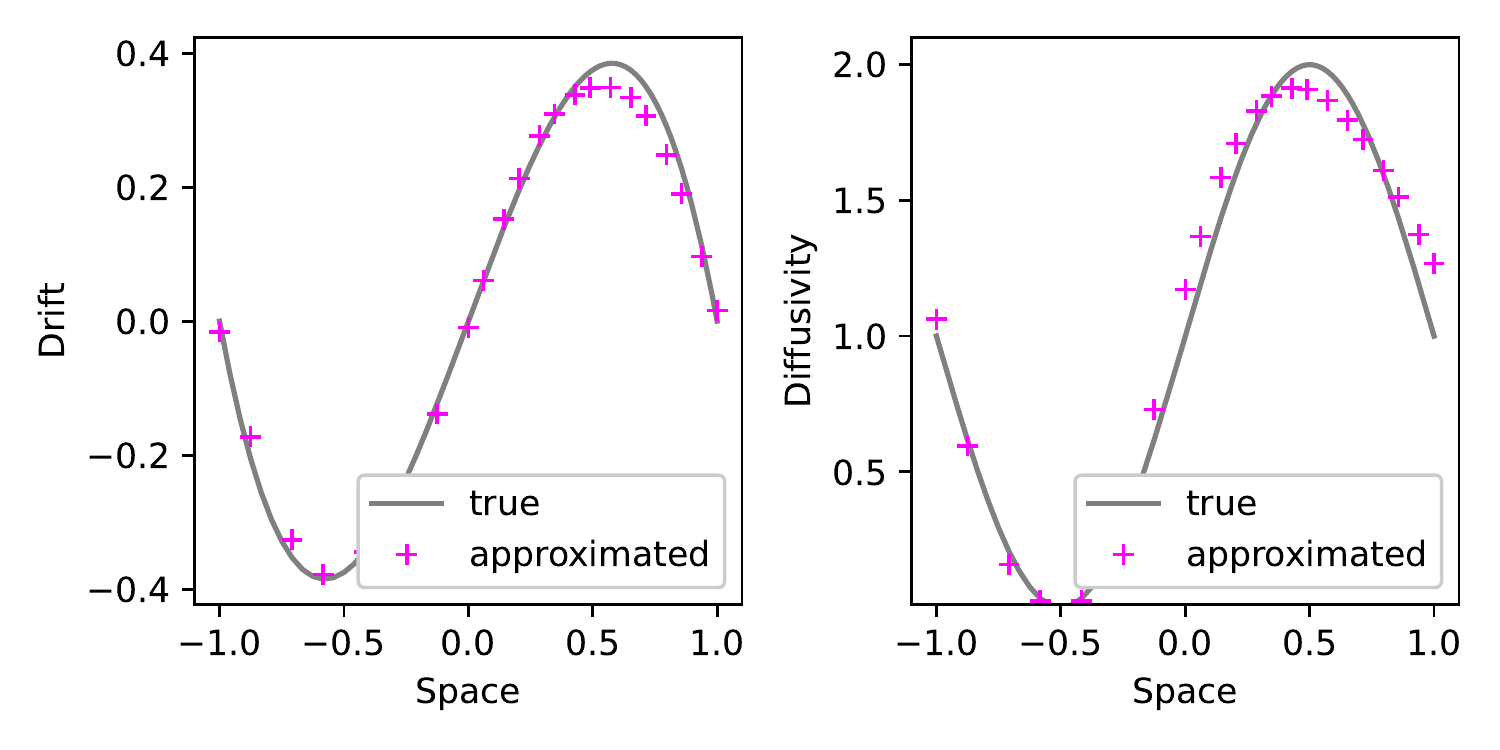}}
  \caption{Multiplicative noise with different diffusion coefficients for $\alpha=1.5$, increased sample size at large noise, the gray lines represent the true coefficients and the magenta plus markers are the estimated drift and diffusion results from our proposed neural networks.}
  \label{fig:1-2-multi-more} 
\end{figure}

\subsection{Comparison with nonlocal Kramers-Moyal formulas}

Next, we compare a novel method mentioned in \cite{Li2021ExtractingSD} which stability parameter $\alpha \in (1,2)$ (due to limitations of the moment generation function method). In this paper, they approximate the stability parameter $\alpha$ and noise intensity $g$ through computing the mean and variance of the amplitude
of increment of the sample paths. Then they estimate the drift coefficient via nonlocal Kramers-Moyal formulas. We keep settings consistent, i.e., $N=50\times1000$, $x_0 \in [-2.5, 2.5]$, the true drift coefficient $f(X_t) = 4X_t-X_t^3$ and the constant noise term $g=1$ except step size. For more details, see \cite{Li2021ExtractingSD}. Figure \ref{compare} shows the results of the comparison, the results are comparable and our work is somewhat competitive.

\begin{figure}[H]
\centering
\includegraphics[width=12cm]{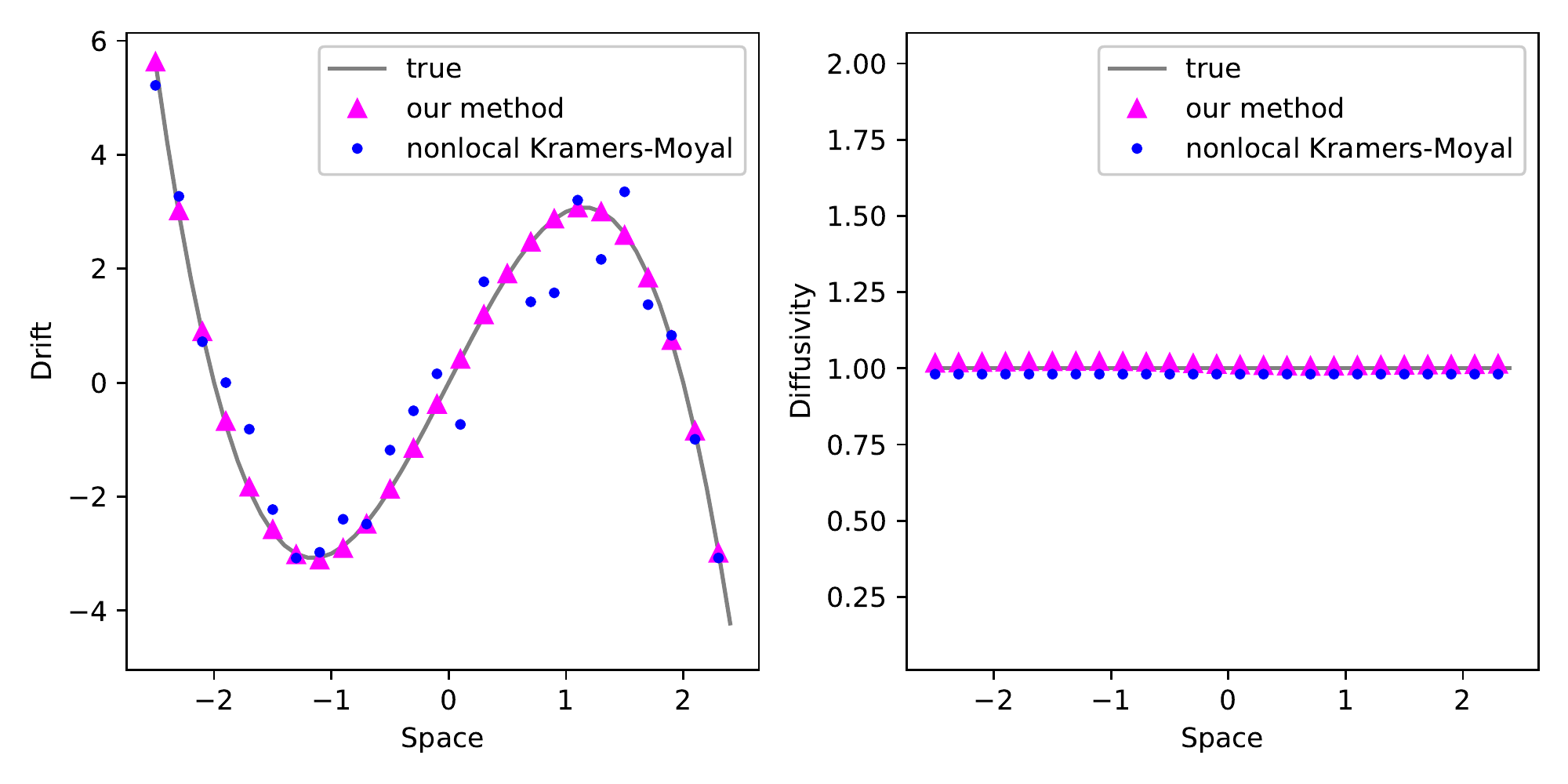}
\caption{Compare with nonlocal Kramers-Moyal formulas with moment generation function by a data-driven method, the gray lines represent the true coefficients, the magenta triangle markers are our results and the blue dots are results of nonlocal Kramers-Moyal method with moment generation function, (left) true and approximated drift coefficients, (right) true and approximated diffusivity coefficients}
\label{compare}
\end{figure}

\subsection{Multi-dimensional experiments}

Finally, we implement the multi-dimensional situation stated in Section \ref{multi_dim}. The neural networks are built in the same way that Section \ref{one dim levy experiments} is. Consider the Maier-Stein model under $\alpha$-stable L\'evy noise in $\mathbb{R}^2$:

\begin{equation} \label{M-S}
    \begin{aligned}
        dX_t &= (X_t - X_t^3 - kX_tY_t^2)dt + dL_{t,1}^\alpha ,\\
        dY_t &= [-(1 + X_t^2)Y_t]dt + dL_{t,2}^\alpha,
    \end{aligned}
\end{equation}
where $L_{t,1}^\alpha$, $L_{t,2}^\alpha$ are two independent one-dimensional symmetric $\alpha$-stable L\'evy motions and $k$ is a positive parameter. We assign $k$ to $1$, as detailed in \cite{Li2020MostPD}. Since we have a different structure of the multi-dimensional L\'evy motion, we no longer compare our methods with those in \cite{Li2021ExtractingSD}. We take sample size $N=(40\ \text{different values of}\ x_0,\  y_0\ \text{respectively}) \times1000$, $h = 0.5$, $(x_0,y_0) \in [-1,1] \times [-1,1]$ and $\alpha=1.5$. We compare the true coefficients and the learned coefficients in the form of heat maps (Figure \ref{fig:M-S MODEL}). Intuitively, the method we proposed reaches a relatively good estimation, and $\mathcal{L}^2$ errors support this view.

Moreover, we give a simple linearly coupled two-dimensional SDE with multiplicative noise in Appendix \ref{2d-multi} as a validation for multi-dimension with multiplicative noise. However, with the dimension of the state space increasing, the sample size needs to grow dramatically, causing issues such as high computational burden and memory limit exceeded. This inspires us to look for more innovative algorithms with sample efficiency and relatively lower computational cost.

\begin{figure}[H]
  \centering
  \subfigure[$f_1(X_t,Y_t)=X_t - X_t^3 - X_tY_t^2$, (left) true drift coefficient, (right) approximated drift coefficient]{
    \label{fig:subfig:2d-ms-drift1} 
    \includegraphics[scale=0.38]{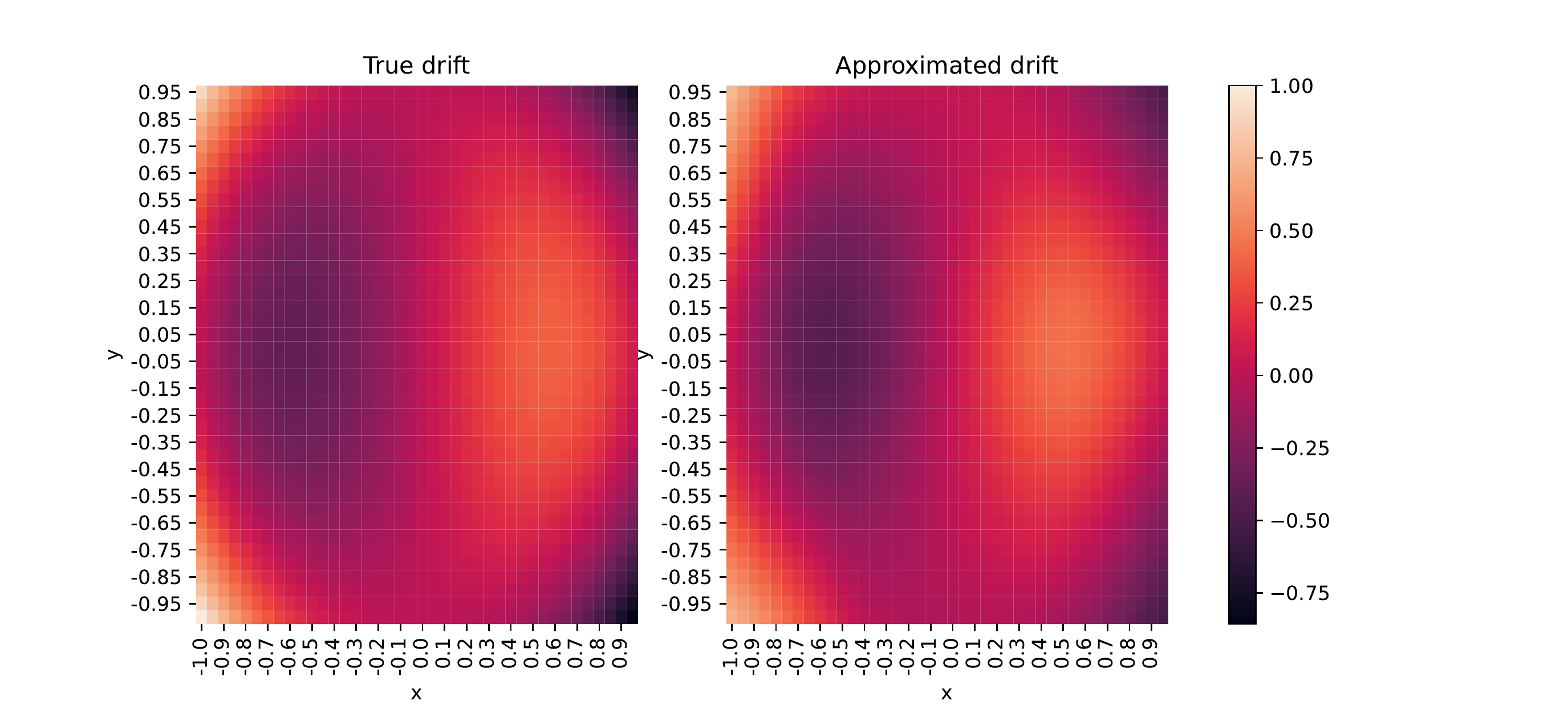}}
  \hspace{0.1in} 
  \subfigure[$f_2(X_t,Y_t)=-(1 + X_t^2)Y_t$, (left) true drift coefficient, (right) approximated drift coefficient]{
    \label{fig:subfig:2d-ms-drift2} 
    \includegraphics[scale=0.38]{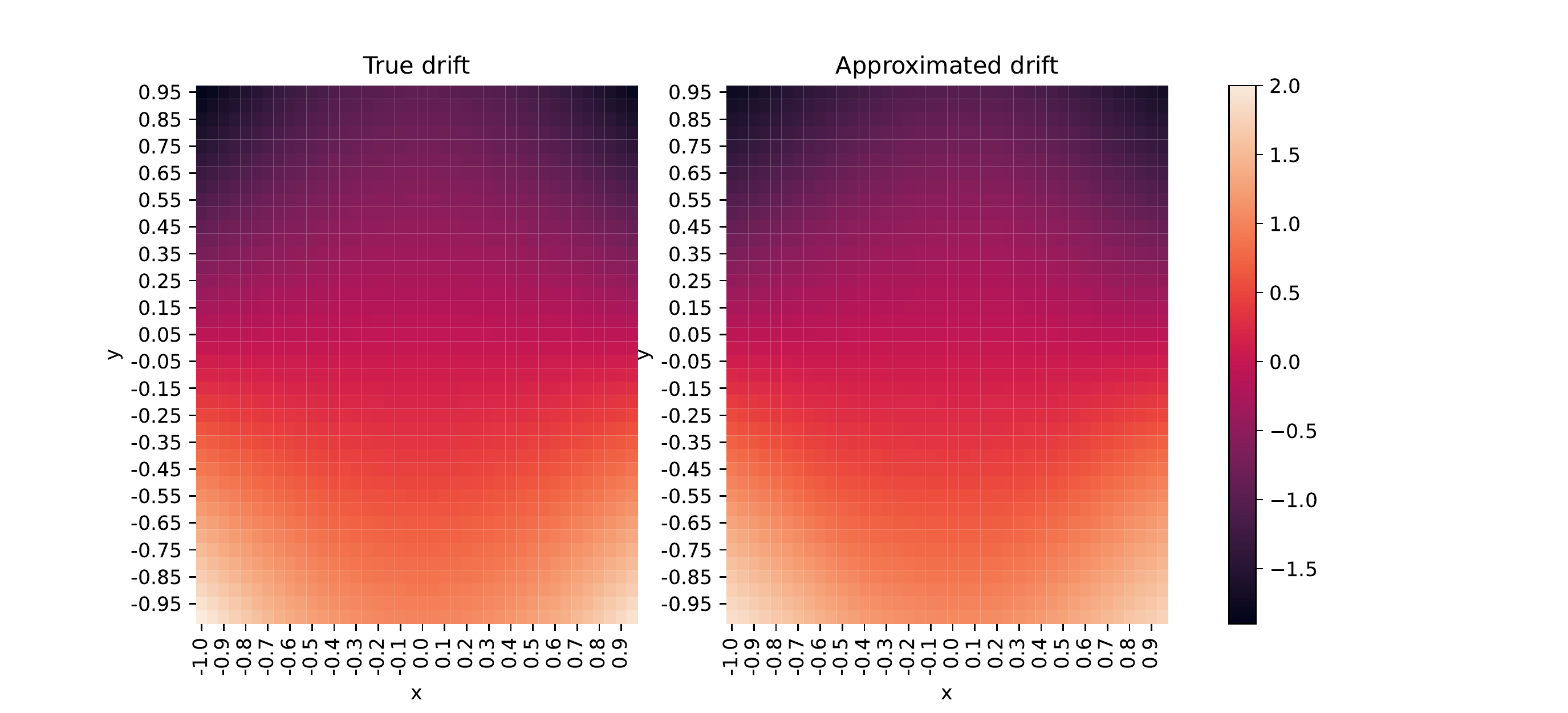}}
  \hspace{0.1in} 
  \subfigure[$g_1(X_t,Y_t)=1$, (left) true diffusivity coefficient, (right) approximated diffusivity coefficient]{
    \label{fig:subfig:2d-ms-diffusion1} 
    \includegraphics[scale=0.38]{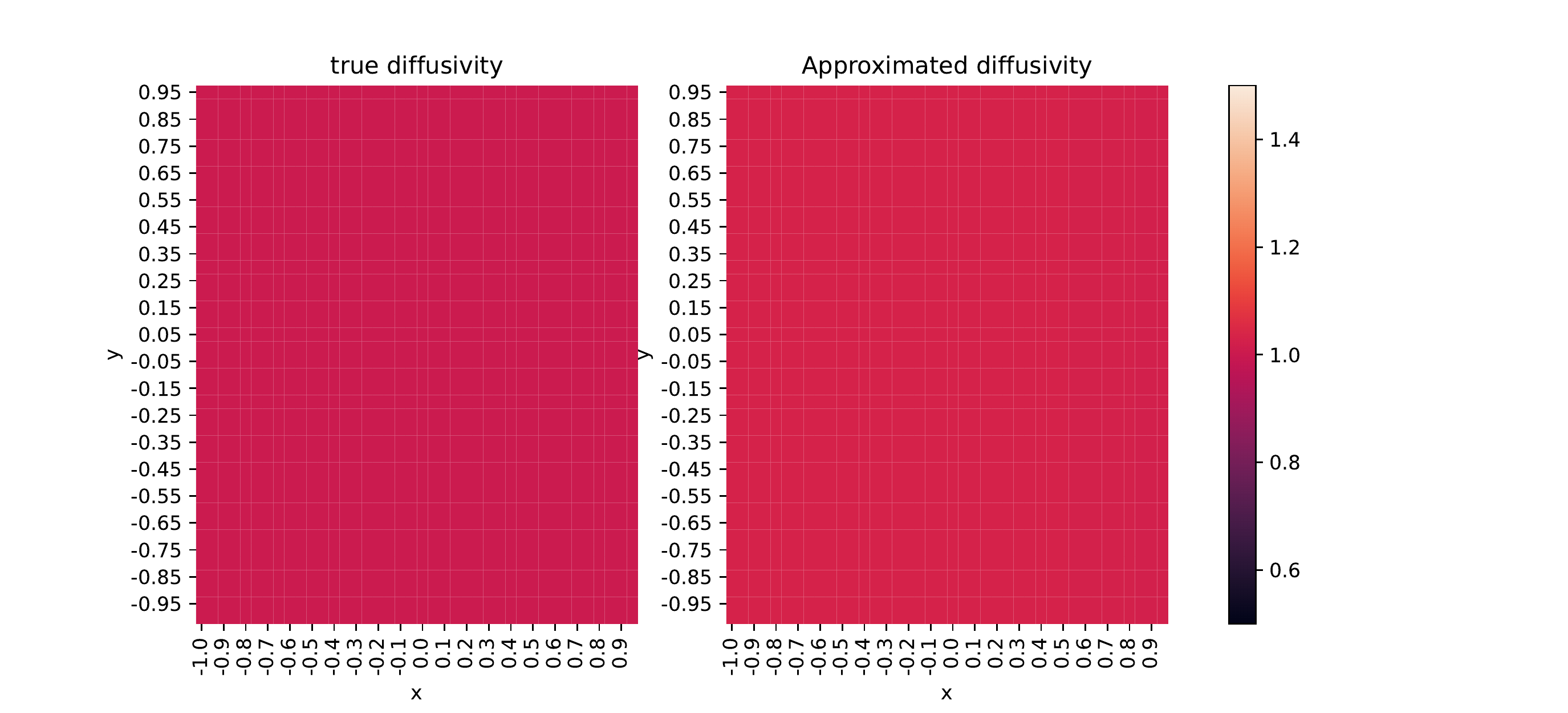}}
  \hspace{0.1in} 
  \subfigure[$g_2(X_t,Y_t)=1$, (left) true diffusivity coefficient, (right) approximated diffusivity coefficient]{
    \label{fig:subfig:2d-ms-diffusion2} 
    \includegraphics[scale=0.38]{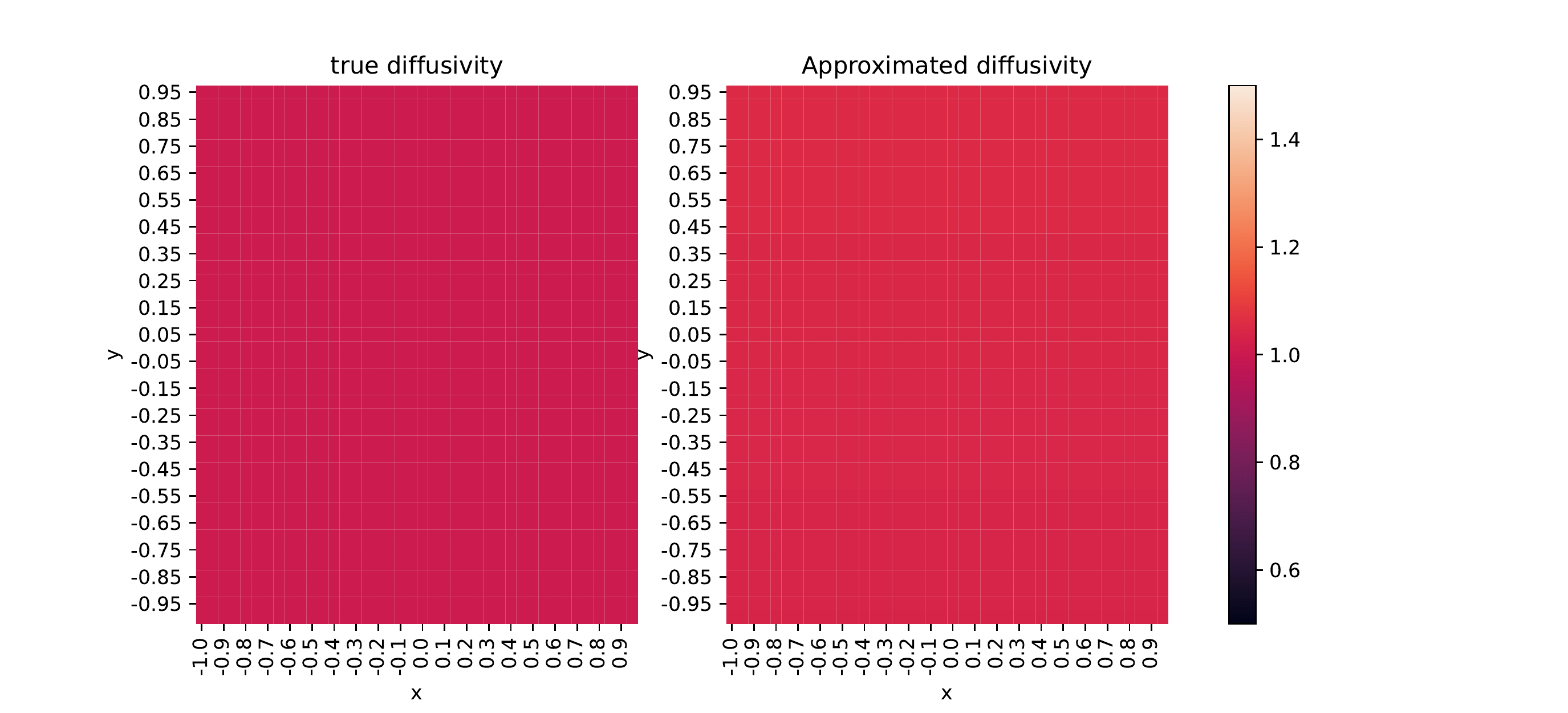}}
  \caption{The drift coefficients (true and estimated)  and diffusion coefficients (true and estimated) of the Maier-Stein model, presented as heat maps.}
  \label{fig:M-S MODEL} 
\end{figure}

\section{Summary} \label{summary}

In this article, we have devised a novel method based on neural networks to identify stochastic dynamics from snapshot data. In particular, we start with a simple identification of the stability parameter and then estimate the drift coefficient via the network. After obtaining an estimation of the drift coefficient, we calculate the diffusion coefficient directly for Cauchy noise and estimate it through another neural network with log-likelihood estimation for $\alpha$-stable L\'evy noise. Numerical experiments on both additive and multiplicative noise illustrate the accuracy and effectiveness of our method.

Compared with nonlocal Kramers-Moyal method with moment generation function, this method has the advantage that it doesn't restrict the value range of the stability parameter and allow multiplicative noise. Even in the face of large noise, simply increasing the local sample size can improve the estimation. 

Moreover, we use Maier-Stein model as an illustration for multi-dimensional SDEs under both additive and multiplicative L\'evy noise.

We give some brief error analysis, and the error of this method is mainly derived from the following parts: errors from neural networks, maximum likelihood estimation, and discretization method. The first two types of errors can be thought as optimization error and generalization error which are determined by the model adopted and the sample size. The choice of hyperparameters also tends to be subjective through personal experience. The last error which can be regarded as approximation error depends on the numerical discretization method and replacing the numerical method also results in a change in the probability distribution, the likelihood function, and other details.  We have tried to balance the various errors so that the result is what we want.

Finally, there still exist some limitations on the applications of this method. For example, although the maximum log-likelihood estimation is precise, the training process is slow. In addition, We will try to learn the probability density function in more general multi-dimensional cases.

\section*{Acknowledgements}
We would like to thank Lingyu Feng, Wei Wei, Min Dai, Yang Li for helpful discussions. This work was supported by National Natural Science Foundation of China (NSFC) 12141107.

\section*{Data Availability}
The data that support the findings of this study are openly available in
GitHub \url{https://github.com/Fangransto/learn-alpha-stable-levy.git}.

\newcommand{\etalchar}[1]{$^{#1}$}

\begin{appendices}

\section{The $\alpha$-stable random variables (1-D) } \label{alpha rv}

\renewcommand\theequation{\Alph{section}\arabic{equation}}
 
\setcounter{equation}{0}

The $\alpha$-stable distributions are a rich class of probability distributions that allow skewness and heavy tails and have many intriguing mathematical properties. According to definition \cite{duan2015introduction}, a random variable $X$ is called a stable random variable if it is a limit in distribution of a scaled sequence $(S_n-b_n)/a_n$, where $S_n = X_1 + \dots + X_n$, $X_i$ are some independent identically distributed random variables and $a_n > 0$ and $b_n$ are some real sequences.

The distribution of a stable random variable is denoted as $S_\alpha(\sigma,\beta,\gamma)$. The $\alpha$-stable distribution requires four parameters for complete description: an index of stability $\alpha \in (0,2]$ also called the tail index, tail exponent or characteristic exponent, a skewness parameter $\beta \in [-1,1]$, a scale parameter $\sigma > 0$ and a location parameter $\gamma \in \mathbb{R}^1$. As mentioned in \cite{duan2015introduction}, \cite{borak2005stable}, \cite{mcculloch1986simple} and \cite{mittnik1999maximum},  the $\alpha$-stable random variables $X \sim S_\alpha(\sigma,\beta,\gamma)$ whose densities lack closed form formulas for all but three distributions, see Table \ref{tab:closed form}.

\begin{table}[H]
\caption{\label{tab:closed form}Closed form formulas for three distributions.}
\centering
\begin{tabular}{|c|c|c|}
\hline
Name & Parameters & Probability density function \\\hline
normal distribution & $\alpha=2,\  \beta=0$ & $p(x)=\dfrac{1}{\sqrt{(2\pi)\times(2\sigma^2)}}\exp(-\dfrac{(x-\gamma)^2}{4\sigma^2})$\\\hline
Cauchy distribution & $\alpha=1,\  \beta=0$ & $p(x)=\dfrac{\sigma}{\pi[(x-\gamma)^2+\sigma^2]}$\\\hline
L\'evy distribution & $\alpha=1/2,\  \beta=1$ & 

$p(x)=\left\{
    \begin{array}{lr}
    \sqrt{\dfrac{\sigma}{2\pi}}(x-\gamma)^{-\frac{3}{2}}\exp\big[-\dfrac{\sigma}{2(x-\gamma)}\big],\  for\ x > \gamma\\
    \\
    0,\   for\ x \le \gamma\\
    \end{array}
    \right.
$
\\\hline
\end{tabular}

\end{table}

\subsection{Probability density functions and characteristic functions} \label{alpha pdf cf}

Here we use three explicit expressions to express the probability density functions of standard symmetric $\alpha$-stable random variables ($X \sim S_\alpha(1, 0 , 0)$, $\alpha \in (0,2)$, $\alpha \ne 1$):

\begin{itemize}
\item Represented as infinite series \cite{duan2015introduction}, 
\begin{equation}
    p(x;\alpha, \beta=0, \sigma=1, \gamma=0) =
    \left\{
    \begin{array}{lr}
    \frac{1}{\pi x}\sum_{k=1}^{\infty}\frac{(-1)^{k-1}}{k!}\Gamma(\alpha k +1)|x|^{-\alpha k}\sin(\frac{k\alpha \pi}{2}),\  for\ x \ne 0,\ 0< \alpha < 1, \\
    \\
    \frac{1}{\pi}\int_0^{\infty}e^{-u^{\alpha}}du,\  for\ x = 0,\ 0< \alpha < 1,\\
    \\
    \frac{1}{\pi\alpha}\sum_{k=0}^{\infty}\frac{(-1)^k}{2k!}\Gamma(\frac{2k+1}{\alpha})x^{2k},\  for\ 1 <\alpha < 2.\\
    \end{array}
    \right.
    \label{pdf1 }
\end{equation}
\item By virtue of characteristic functions and the Fourier transform \cite{borak2005stable}, 
\begin{equation}
    p(x;\alpha, \beta=0, \sigma=1, \gamma=0) = \frac{1}{2\pi}\int_{-\infty}^{\infty}e^{-ixt}{\phi(t;\alpha, \beta=0, \sigma=1, \gamma=0)}dt,
    \label{pdf2 }
\end{equation}
\item Zolotarev formulas \cite{borak2005stable},
\begin{equation}
    p(x;\alpha, \beta=0, \sigma=1, \gamma=0)=
    \left\{
    \begin{array}{lr}
    \dfrac{\alpha x^{\frac{1}{\alpha-1}}}{\pi|\alpha-1|}\int_0^{\frac{\pi}{2}}V(\theta;\alpha,\beta=0)\exp\{-x^{\frac{1}{\alpha-1}}V(\theta;\alpha,\beta=0)\}d\theta,\  for\ x > 0,\\
    \\
    \dfrac{\Gamma(1+\frac{1}{\alpha})}{\pi},\  for\ x = 0,\\
    \\
    p(-x;\alpha, -\beta=0, \sigma=1, \gamma=0),\  for\ x < 0,\\
    \end{array}
    \right.
    \label{pdf3 }
\end{equation}
where $V(\theta;\alpha,\beta=0)=\Big(\dfrac{\cos\theta}{\sin\alpha\theta}\Big)^{\frac{\alpha}{\alpha-1}}\cdot\dfrac{\cos\{(\alpha-1)\theta\}}{\cos\theta}$.
\end{itemize}

Notice that, we mention $\alpha$-stable random variables' characteristic functions in (\ref{pdf2 }).  The most popular parameterization of the characteristic function of $X \sim S_\alpha(\sigma, \beta , \gamma)$ is given by:

\begin{equation}
    \mathbb{E}\exp(i\left \langle t,X\right \rangle) = \ln\phi(t)=
    \left\{
    \begin{array}{lr}
    -\sigma^{\alpha}|t|^{\alpha}\{1-i\beta\sign t\tan(\frac{\pi\alpha}{2})\}+i\gamma t\  \alpha \neq 1,\\
    \\
    -\sigma |t|\{1+i\beta\sign t\frac{2}{\pi}\ln|t|\}+i\gamma t,\  \alpha=0.\\
    \end{array}
    \right.
    \label{cf1}
\end{equation}
It is often advisable to use Nolan’s parameterization $S^0_{\alpha}(\sigma, \beta , \gamma_0)$:
\begin{equation}
    \ln\phi_0(t)=
    \left\{
    \begin{array}{lr}
    -\sigma^{\alpha}|t|^{\alpha}\{1+i\beta\sign t\tan(\frac{\pi\alpha}{2})[(\sigma|t|)^{1-\alpha}-1]\}+i\gamma_0 t\  \alpha \neq 1,\\
    \\
    -\sigma |t|\{1+i\beta\sign t\frac{2}{\pi}\ln(\sigma|t|)\}+i\gamma_0 t,\  \alpha=0.\\
    \end{array}
    \right.
    \label{cf2}
\end{equation}

The location parameters of the two representations are related by $\gamma = \gamma_0-\beta\sigma\tan(\frac{\pi\alpha}{2})$ for $\alpha \ne 1$ and $\gamma = \gamma_0-\beta\sigma\frac{2}{\pi}\ln\sigma$ for $\alpha = 1$.
 
\subsection{Basic Properties of $\alpha$-stable random variables}

We recall some properties of $\alpha$-stable random variables which may be used in the article.

An important property is that $\alpha$-stable random variables have the characteristic of sharp peak and heavy tail compared with the Gaussian random variables, as the tail estimate decays polynomially:

\begin{equation}
    \lim_{y \rightarrow \infty} y^{\alpha}\mathbb{P}(X>y) = C_\alpha \frac{1+\beta}{2}\sigma^\alpha,
    \qquad
    \lim_{y \rightarrow \infty} y^{\alpha}\mathbb{P}(X<-y) = C_\alpha \frac{1-\beta}{2}\sigma^\alpha,
\end{equation}
where $C_\alpha>1$, $X \sim S_\alpha(\sigma,\beta,\gamma)$, see \cite{Samorodnitsky1995StableNR}, Chapter 1. 

We can generate random numbers from a scalar standard symmetric $\alpha$-stable random variable $S_\alpha(1,0,0)$, for $\alpha \in (0,2)$ \cite{duan2015introduction}:

First generate a uniform random variable $V$ on $(-\frac{\pi}{2},\frac{\pi}{2})$ and an exponential random variable $W$ with parameter $1$. Then a scalar standard symmetric $\alpha$-stable random variable $X \sim \alpha \in (0,2)$ is produced by 

\begin{equation}
X=\frac{\sin \alpha V}{(\cos V)^\frac{1}{\alpha}} \Big\{ \frac{\cos (V-\alpha V)}{W} \Big\} ^\frac{1-\alpha}{\alpha}.
\end{equation}

\section{Drift estimation trick} \label{drift trick}

\renewcommand\theequation{\Alph{section}\arabic{equation}}
 
\setcounter{equation}{0}

Recall the role of the drift coefficients in the SDEs, especially in the Euler-Maruyama discretization method (\ref{E-M LDE}), they assume part of the role of determining the location parameter of $x_1$. However, noise intensity affects the judgment of the location parameter because it allows $x_1$ to move or jump to a very far position in the state space. We have assumed that the $\alpha$-stable L\'evy motions are symmetric in Section \ref{sde setting}, so we can select representative points to eliminate the effects of noise. Specifically, we select the mean, median, and the middle 20\% order statistics for the same $x_0$. A comparison with a no noise differential equation numerical solution ($x_1 = x_0 + f(x_0)h$) is shown in Figure \ref{trick}. Finally, we choose the middle 20\% order statistics for training.

\begin{figure}[H]
\centering
\includegraphics[width=12cm]{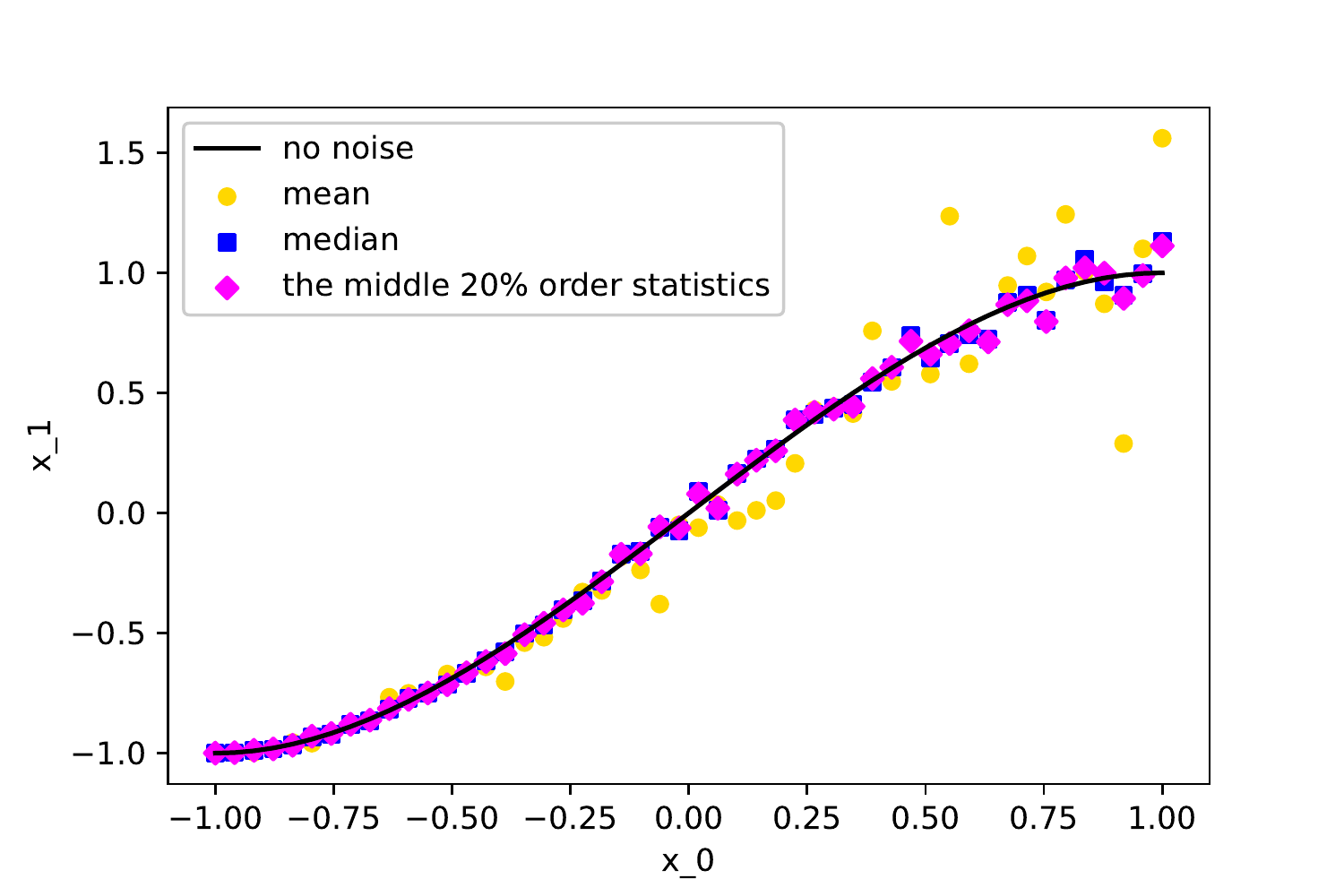}
\caption{Compare mean (gold dots), median (blue squares), the middle 20\% order statistics (magenta diamonds) for the same $x_0$. The results show that the middle 20\% order statistics is the closest case to no noise differential equation numerical solution (black line). Here the true drift coefficient and diffusion coefficient are $f(X_t)=-X_t^3+X_t$ and $g(X_t)= X_t +1$, separately.}
\label{trick}
\end{figure}

\section{More experimental results} \label{other experiments}

\renewcommand\theequation{\Alph{section}\arabic{equation}}
\setcounter{equation}{0}

In this appendix, we show more experimental results, including additive/multiplying noise, different $\alpha$ values, different classes of drift coefficients, a two-dimensional SDE with linear drift coefficients and linear diffusion coefficients.

\subsection{SDEs driven by Cauchy motion}

We maintain the same setting as in Section \ref{cauchy experiments}.

\textbf{Additive noise} 

In addition to what is mentioned in the article, we also consider drift coefficients:  $f(X_t)=-X_t+1$ and $f(X_t)=\sin(X_t)$. Here we take sample size $N=10000$, $h = 0.01$, $x_0 \in [-3,3]$. The comparison results can be seen in Figure \ref{fig:cauchy-add}. The results of the constant noise intensity are also obtained: $\hat{g}_{\theta} = 0.9799 $, $\hat{g}_{\theta} = 0.9593 $.

\begin{figure}[H]
  \centering
  \subfigure[$f(X_t)=-X_t+1$, $\hat{g}_{\theta} = 0.9799 $]{
    \label{fig:subfig:cauchy-lin-add} 
    \includegraphics[scale=0.62]{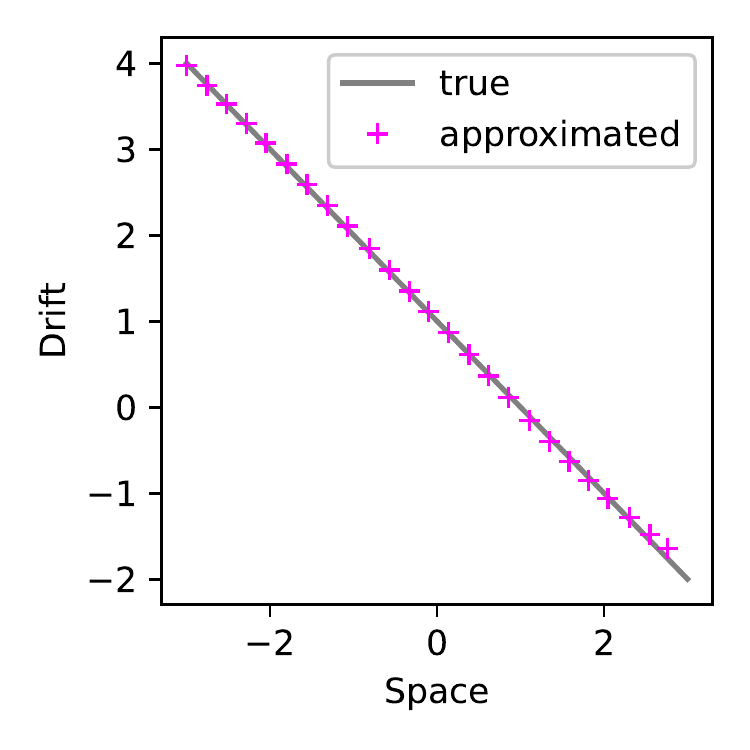}}
  \hspace{0.1in} 
  \subfigure[$f(X_t)=\sin(X_t)$, $\hat{g}_{\theta} = 0.9593 $]{
    \label{fig:subfig:cauchy-sin-add} 
    \includegraphics[scale=0.62]{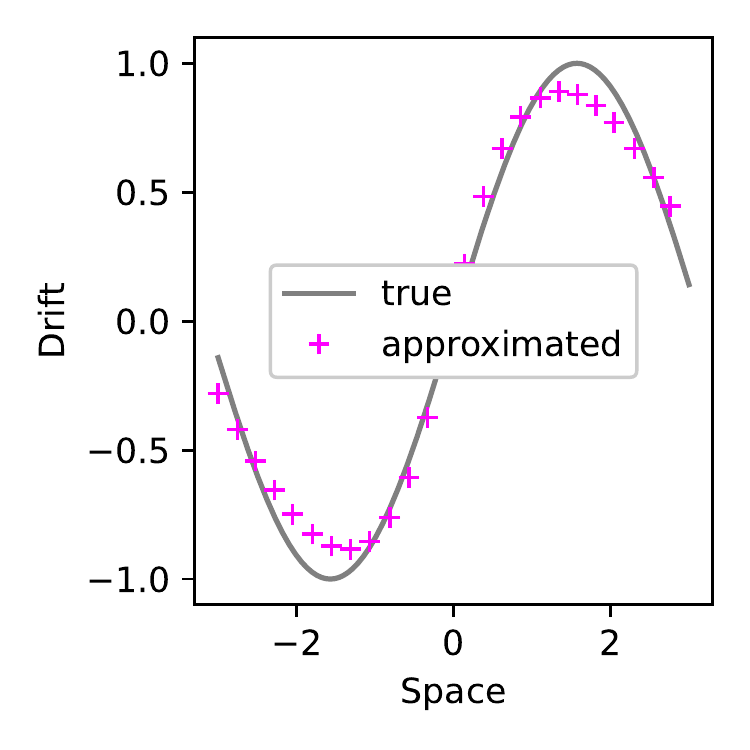}}
  \caption{Cauchy additive noise with different drift coefficients, the gray lines represent the true drift coefficients and the magenta plus markers are the neural networks' results. $\hat{g}_{\theta}$ represent the diffusion coefficients calculated by (\ref{cauchy loss2}).}
  \label{fig:cauchy-add} 
\end{figure}

\textbf{Multiplicative noise}

Replacing additive noise with multiplicative noise, we test the multiplier noise as $g(X_t) = 0.1X_t+0.5$. Following the construction in the previous section except sample size $N = 20$ (different $x_0$) $\times1000$ (sample\ size), a comparison between the learned and true drift/diffusion functions is shown in Figure \ref{fig:cauchy-mul}. 

\begin{figure}[H]
  \centering
  \subfigure[$f(X_t)=-X_t+1$, $g(X_t) = 0.1X_t+0.5$, (left) true and approximated drift coefficients, (right) true and approximated diffusivity coefficients]{
    \label{fig:subfig:cauchy-lin-mul} 
    \includegraphics[scale=0.48]{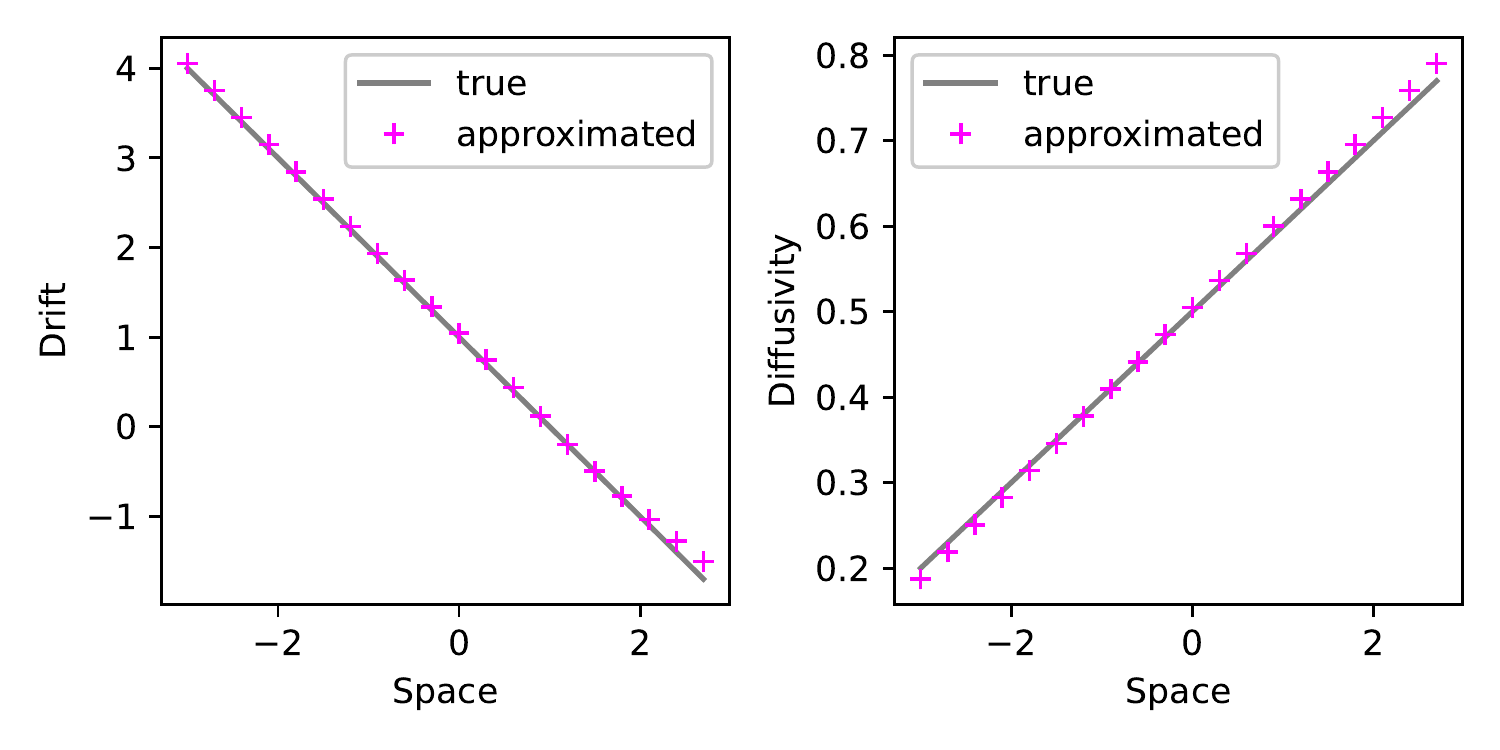}}
  \hspace{0.1in} 
  \subfigure[$f(X_t)=\sin(X_t)$, $g(X_t) = 0.1X_t+0.5$, (left) true and approximated drift coefficients, (right) true and approximated diffusivity coefficients]{
    \label{fig:subfig:cauchy-sin-mul} 
    \includegraphics[scale=0.48]{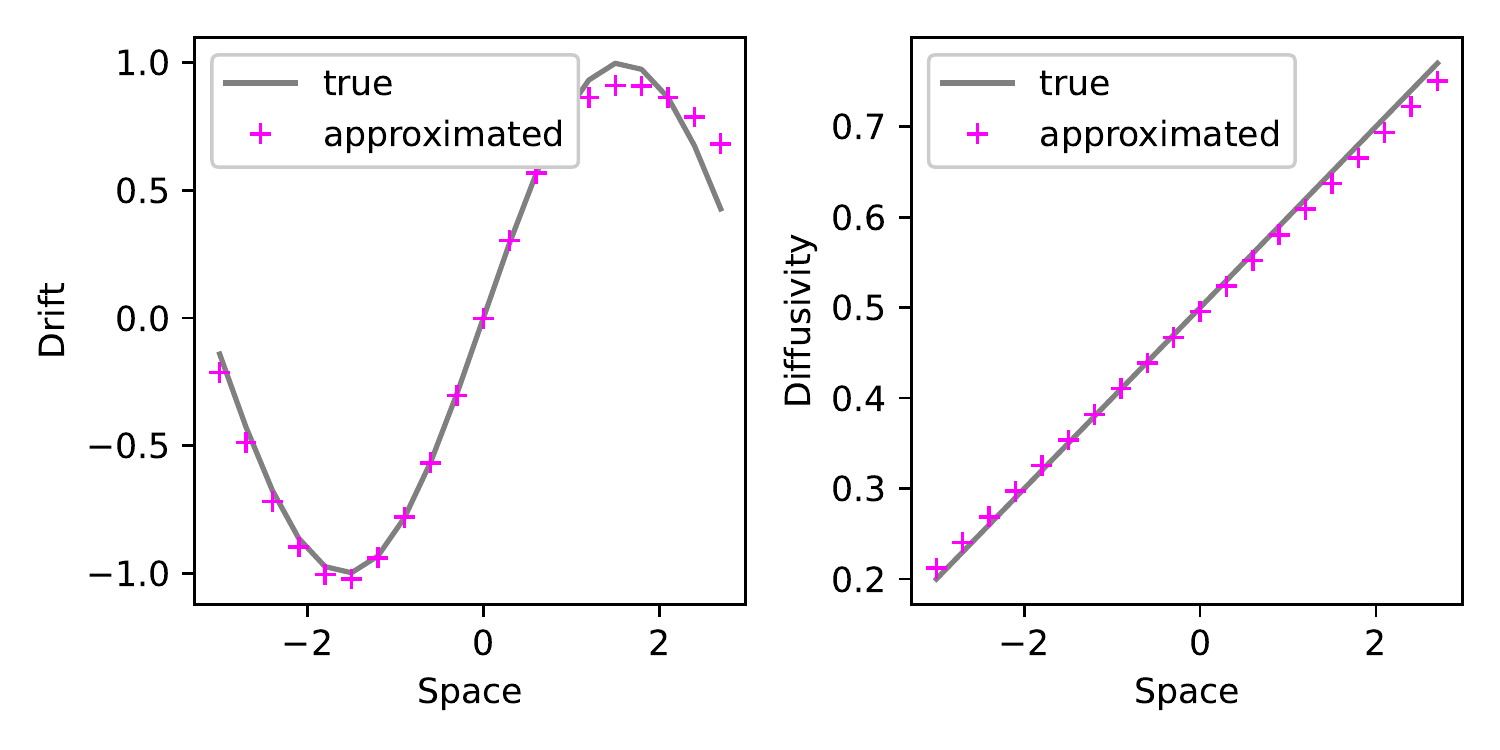}}
  \caption{Cauchy multiplicative noise with different drift coefficients, the gray lines represent the true drift coefficients and true diffusion coefficients, and the magenta plus markers are the estimated drift and diffusion results from our proposed neural networks.}
  \label{fig:cauchy-mul} 
\end{figure}

\subsection{SDEs driven by $\alpha$-stable L\'evy motion}

In this section, neural networks are used twice, separately to estimate drift and diffusion. We continue the construction in Section \ref{one dim levy experiments}.

\textbf{Additive noise}

For $\mathbf{\alpha=1.5}$, we test the drift coefficients $f(X_t) = -X_t+1$, $f(X_t) = -X_t^3 + X_t$ and $f(X_t)=\log (X_t+1.5)-|X_t+1.5|^{1/3}$ in equation (\ref{LDE}), diffusion coefficient $g = 1$. The technique mentioned in appendix \ref{drift trick} is still used. Here we take sample size $N=5\times1000$, $h = 0.1$ for the linear drift coefficient, $N=50\times1000$, $h = 0.5$ for the other two more complex drift coefficients, both cases $x_0 \in [-1,1]$.
The true drift coefficients and the fitting results of the neural network are compared as shown in Figure \ref{fig:1-2-add}.

\begin{figure}[H]
  \centering
  \subfigure[$f(X_t)=-X_t+1$, $g(X_t) = 1$, (left) true and approximated drift coefficients, (right) true and approximated diffusivity coefficients]{
    \label{fig:subfig:1-2-lin-add} 
    \includegraphics[scale=0.48]{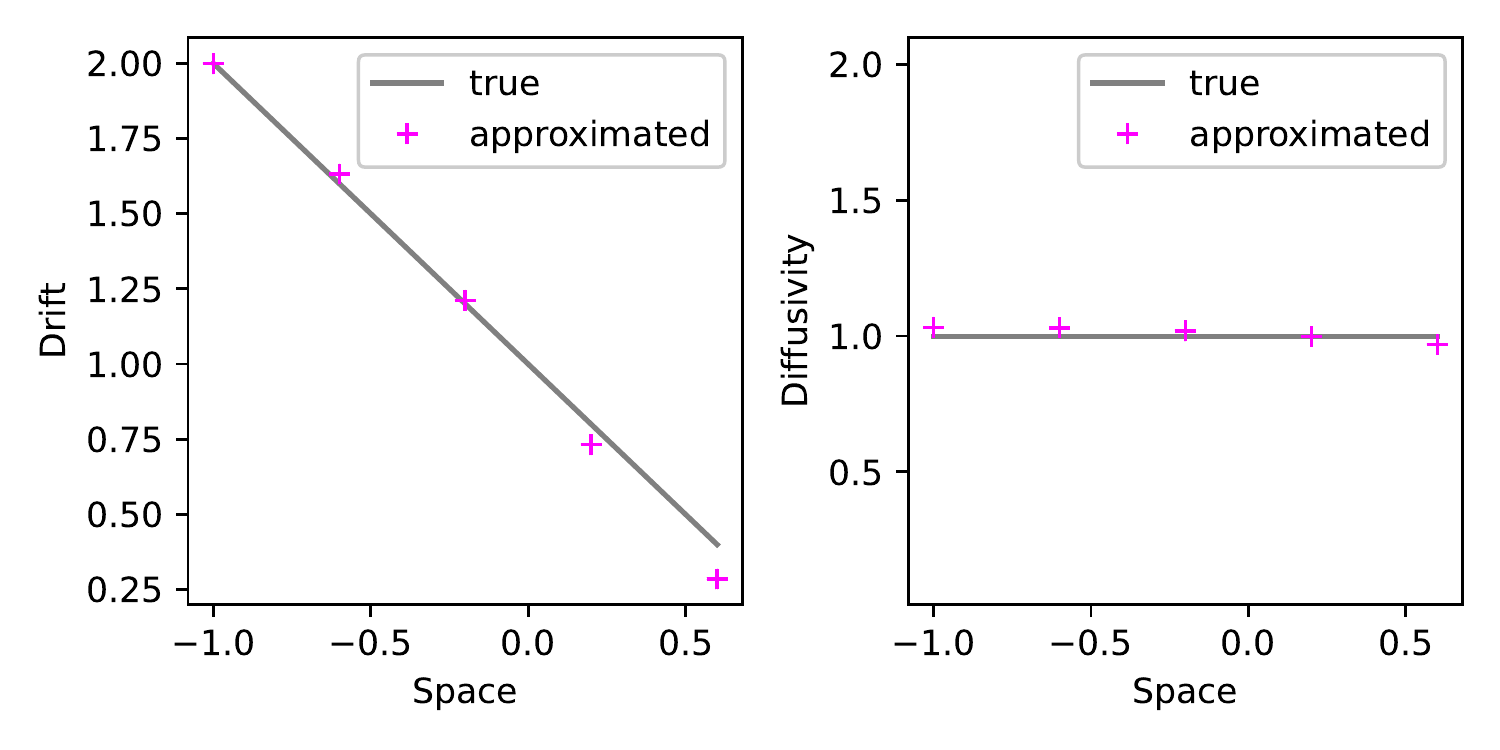}}
  \hspace{0.1in} 
  \subfigure[$f(X_t)=-X_t^3+X_t$, $g(X_t) = 1$, (left) true and approximated drift coefficients, (right) true and approximated diffusivity coefficients]{
    \label{fig:subfig:1-2-dw-add} 
    \includegraphics[scale=0.48]{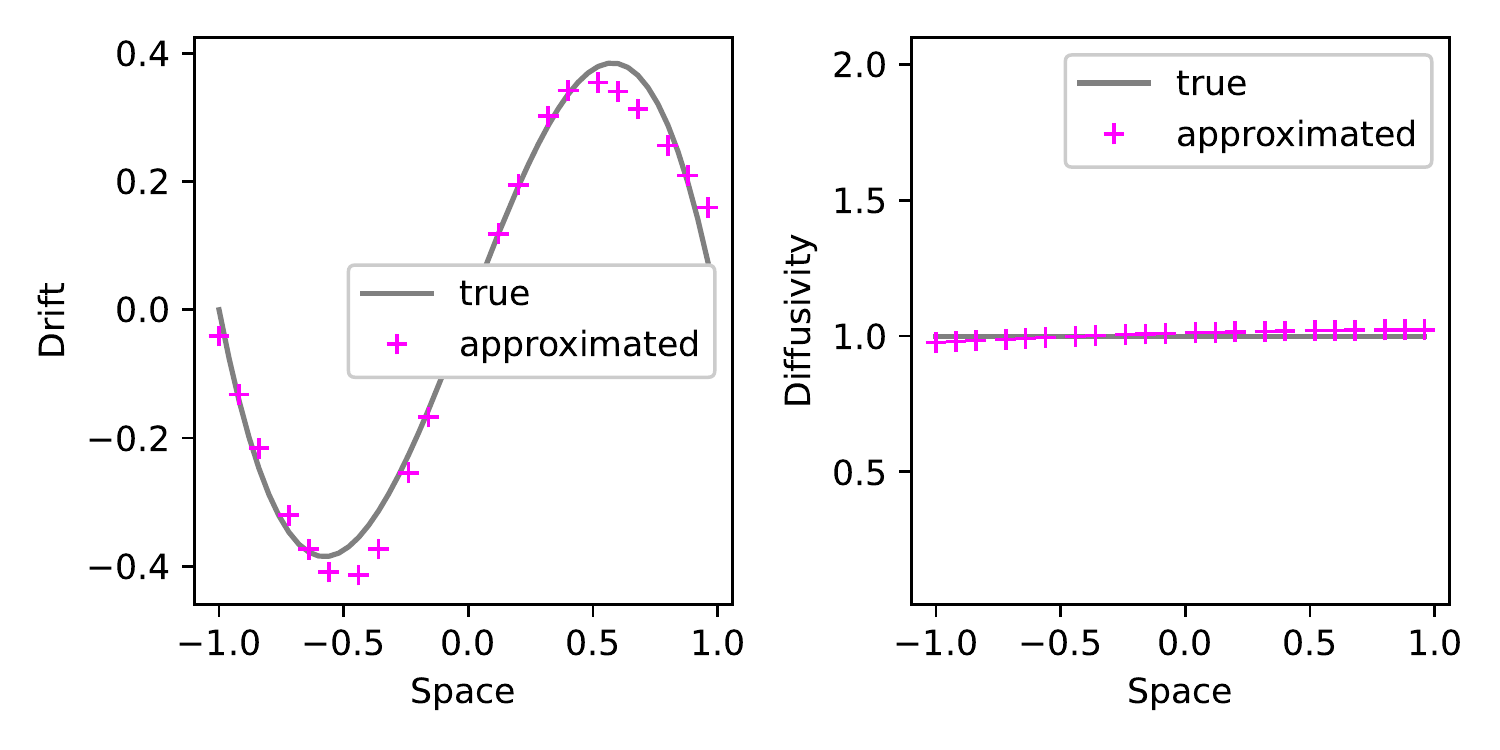}}
  \hspace{0.1in} 
  \subfigure[$f(X_t)= \log (X_t+1.5)-|X_t+1.5|^{1/3}$, $g(X_t) = 1$, (left) true and approximated drift coefficients, (right) true and approximated diffusivity coefficients]{
    \label{fig:subfig:1-2-complex-add} 
    \includegraphics[scale=0.48]{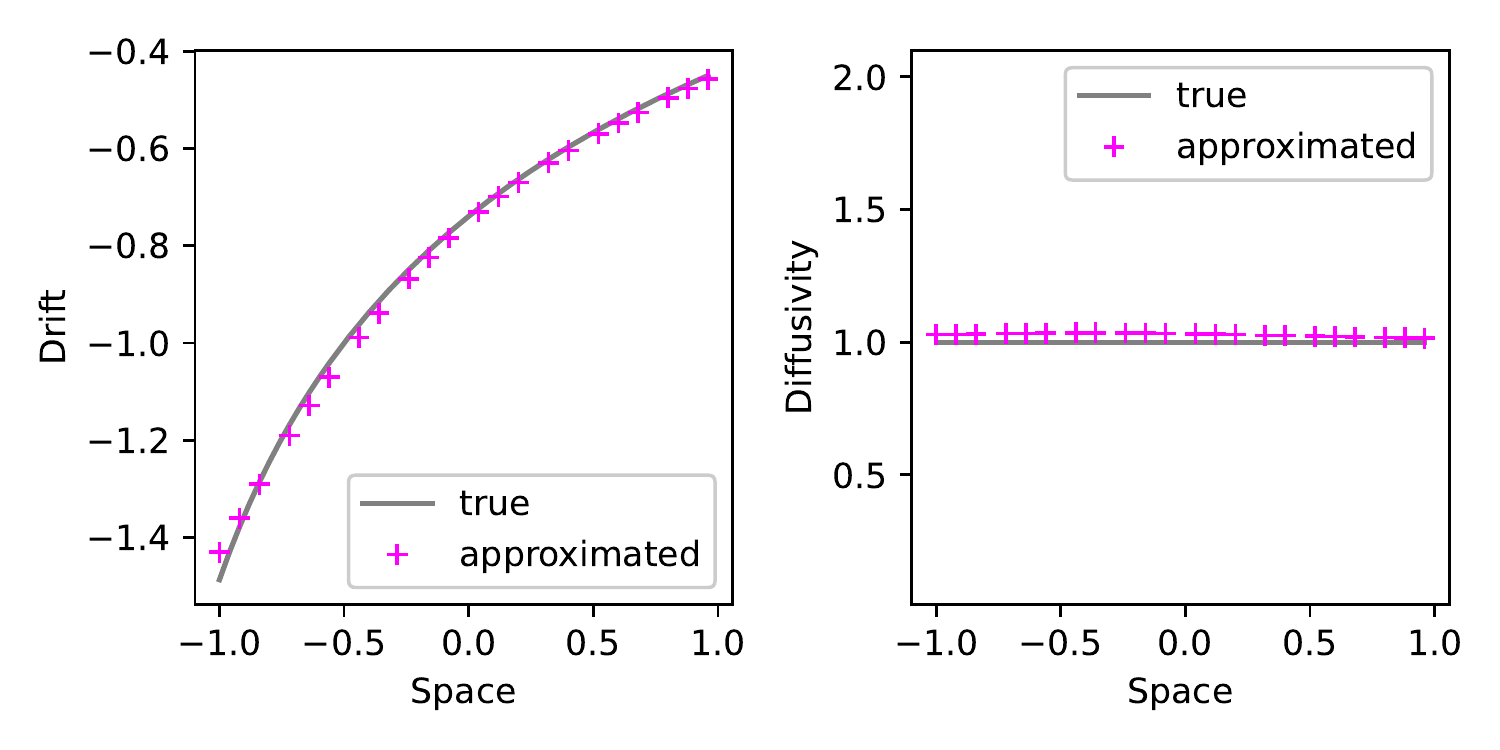}}
  \caption{Additive noise with different drift coefficients for $\alpha=1.5$, the gray lines represent the true drift coefficients and true diffusion coefficients, and the magenta plus markers are the estimated drift and diffusion results from our proposed neural networks.}
  \label{fig:1-2-add} 
\end{figure}

For $\mathbf{\alpha=0.5}$, hyperparameter settings are almost identical. The comparison results are shown in the Figure \ref{fig:0-1-add}.

\begin{figure}[H]
  \centering
  \subfigure[$f(X_t)=-X_t+1$, $g(X_t) = 1$, (left) true and approximated drift coefficients, (right) true and approximated diffusivity coefficients]{
    \label{fig:subfig:0-1-lin-add} 
    \includegraphics[scale=0.48]{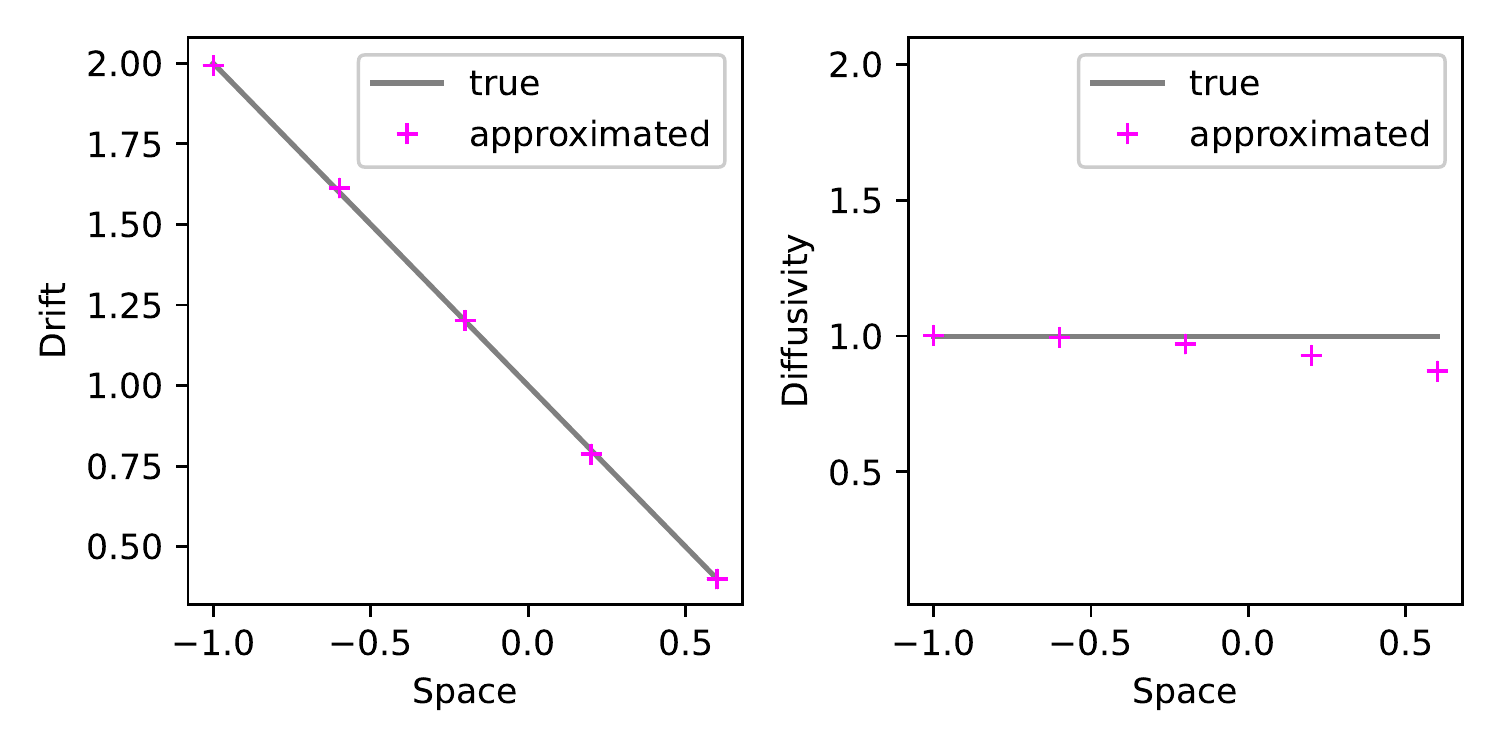}}
  \hspace{0.1in} 
  \subfigure[$f(X_t)=-X_t^3+X_t$, $g(X_t) = 1$, (left) true and approximated drift coefficients, (right) true and approximated diffusivity coefficients]{
    \label{fig:subfig:0-1-dw-add} 
    \includegraphics[scale=0.48]{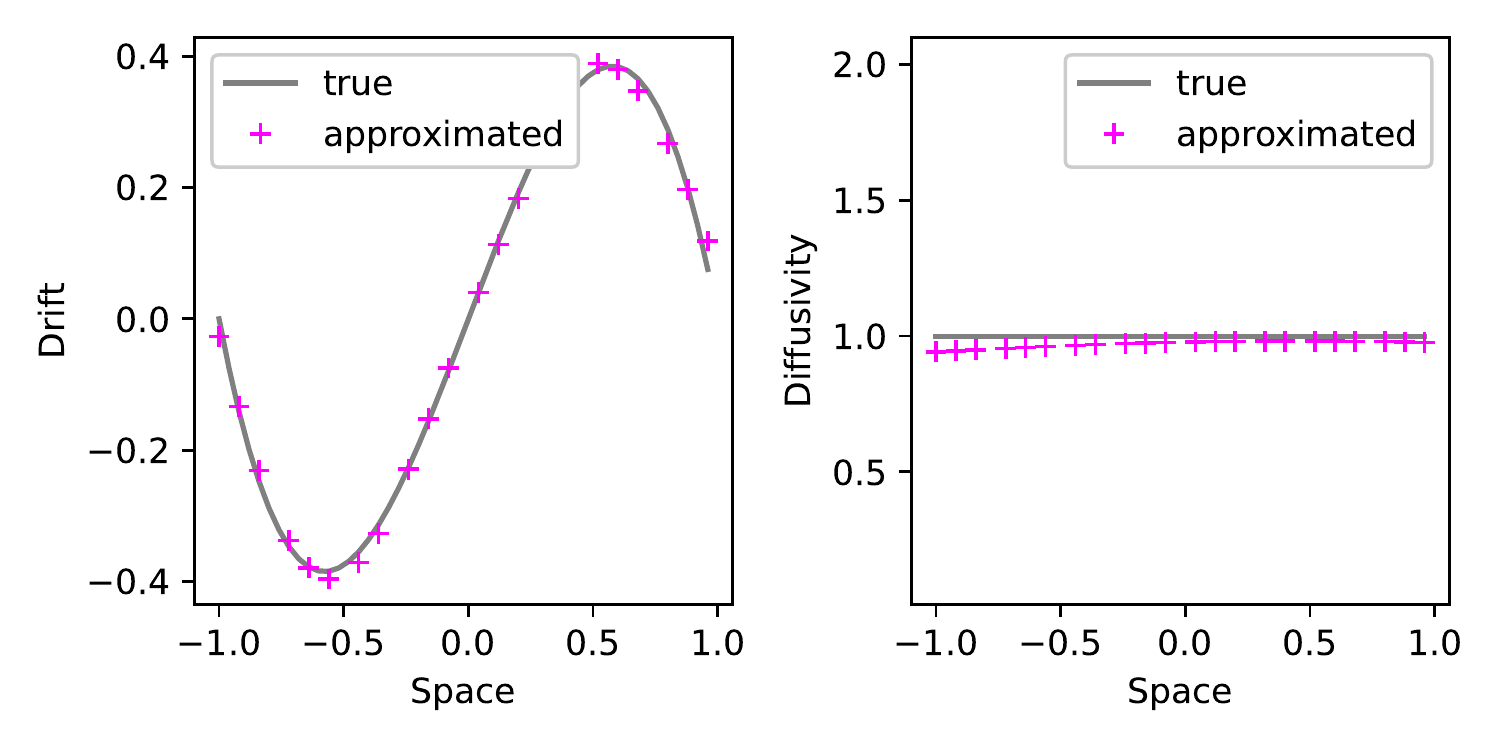}}
  \caption{Additive noise with different drift coefficients for $\alpha=0.5$, the gray lines represent the true drift coefficients and true diffusion coefficients, and the magenta plus markers are the estimated drift and diffusion results from our proposed neural networks.}
  \label{fig:0-1-add} 
\end{figure}

\textbf{Multiplicative noise}

For $\mathbf{\alpha=0.5}$, we test (1) $f(X_t) = -X_t+1$ and $g(X_t) = X_t + 1$, $N=5\times1000$, $h = 0.1$; (2) $f(X_t) = -X_t^3+X_t$ and $g(X_t) = X_t + 1$, $N=75\times1000$, $h = 0.5$. The comparison results are shown in the Figure \ref{fig:0-1-mul}.

\begin{figure}[H]
  \centering
  \subfigure[$f(X_t)=-X_t+1$, $g(X_t) = X_t+1$, (left) true and approximated drift coefficients, (right) true and approximated diffusivity coefficients]{
    \label{fig:subfig:0-1-dw-mul1} 
    \includegraphics[scale=0.48]{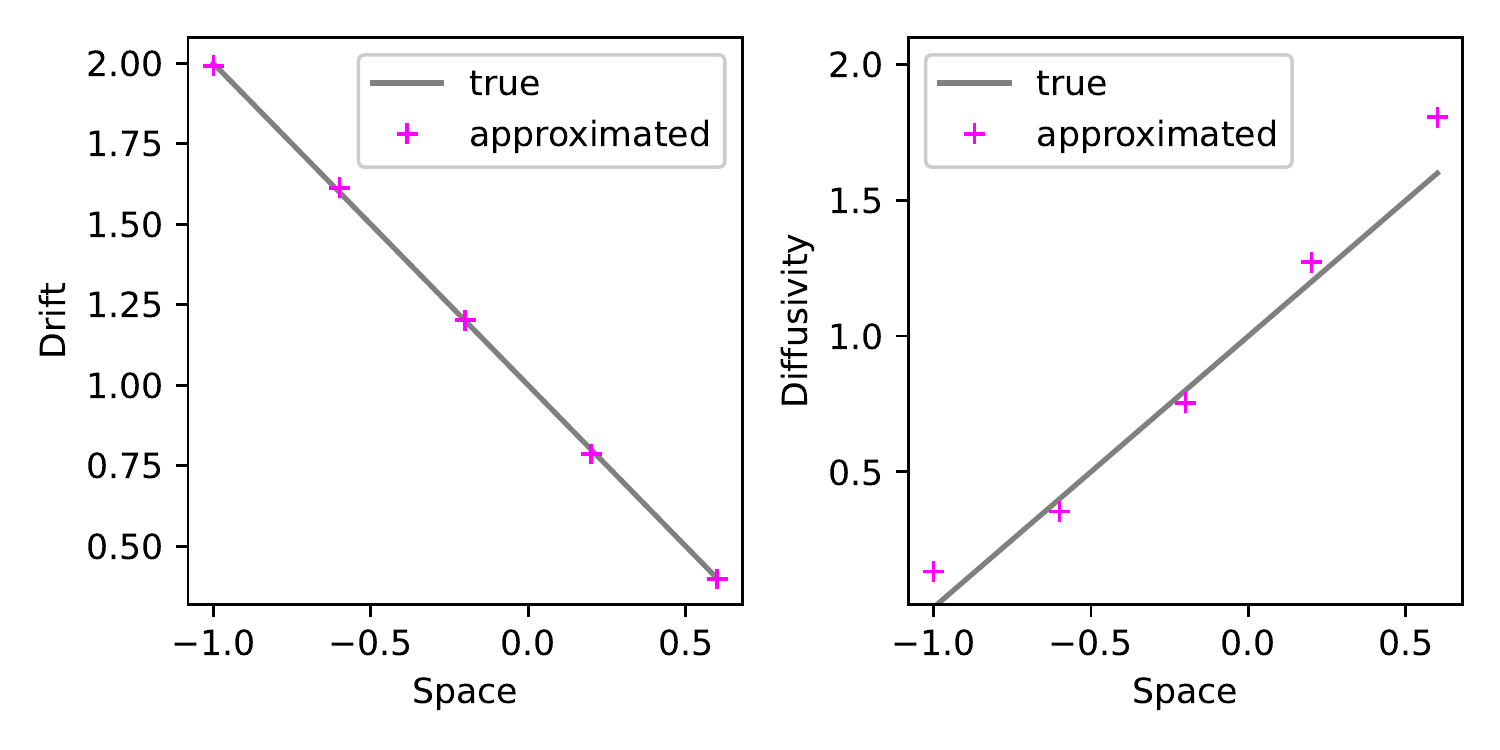}}
  \hspace{0.1in} 
  \subfigure[$f(X_t)=-X_t^3+X_t$, $g(X_t) = X_t + 1$, (left) true and approximated drift coefficients, (right) true and approximated diffusivity coefficients]{
    \label{fig:subfig:0-1-dw-mul2} 
    \includegraphics[scale=0.48]{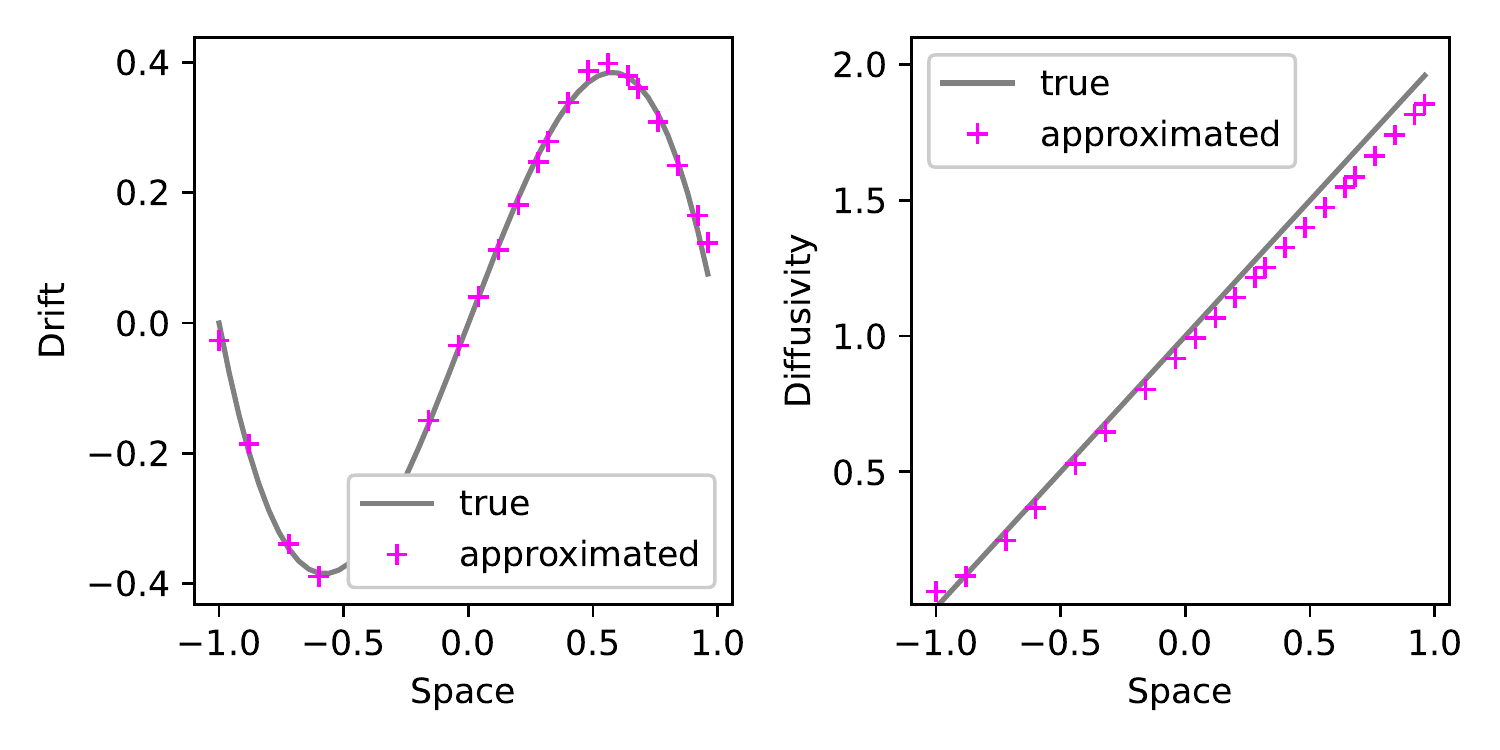}}
  \caption{Multiplicative noise with different diffusion coefficients for $\alpha=0.5$, the gray lines represent the true coefficients and the magenta plus markers are the neural networks' results.}
  \label{fig:0-1-mul} 
\end{figure}

\subsection{Another multi-dimensional experiment} \label{2d-multi}

In this experiment, the hyperparameter settings are consistent with section \ref{one dim levy experiments}. To demonstrate our capacity to learn multiplicative coupled noise, we examined the approximation of the following SDE:

\begin{equation} \label{2d-lin-lin}
    \begin{aligned}
        dX_t &= (X_t+Y_t)dt + (0.5Y_t+1) dL_{t,1}^\alpha ,\\
        dY_t &= (4X_t-2Y_t)dt + (0.5X_t+1) dL_{t,2}^\alpha,
    \end{aligned}
\end{equation}
where $L_{t,1}^\alpha$ and $L_{t,2}^\alpha$ are two independent one-dimensional $\alpha$-stable L\'evy motions. Here we take $\alpha=1.5$, $N=(5\ \text{different values of}\ x_0,\  y_0\ \text{respectively})\times 1000$, $h=0.5$, $(x_0,y_0)\in [-1,1] \times [-1,1]$. The true drift and diffusion coefficients, as well as the approximate results, are also presented as heat maps (Figure \ref{fig:l2d-in-lin MODEL}). As can be seen, our method is still workable.

\begin{figure}[H]
  \centering
  \subfigure[$f_1(X_t,Y_t)=X_t+Y_t$, (left) true drift coefficient, (right) approximated drift coefficient]{
    \label{fig:subfig:2d-lin-drift1} 
    \includegraphics[scale=0.38]{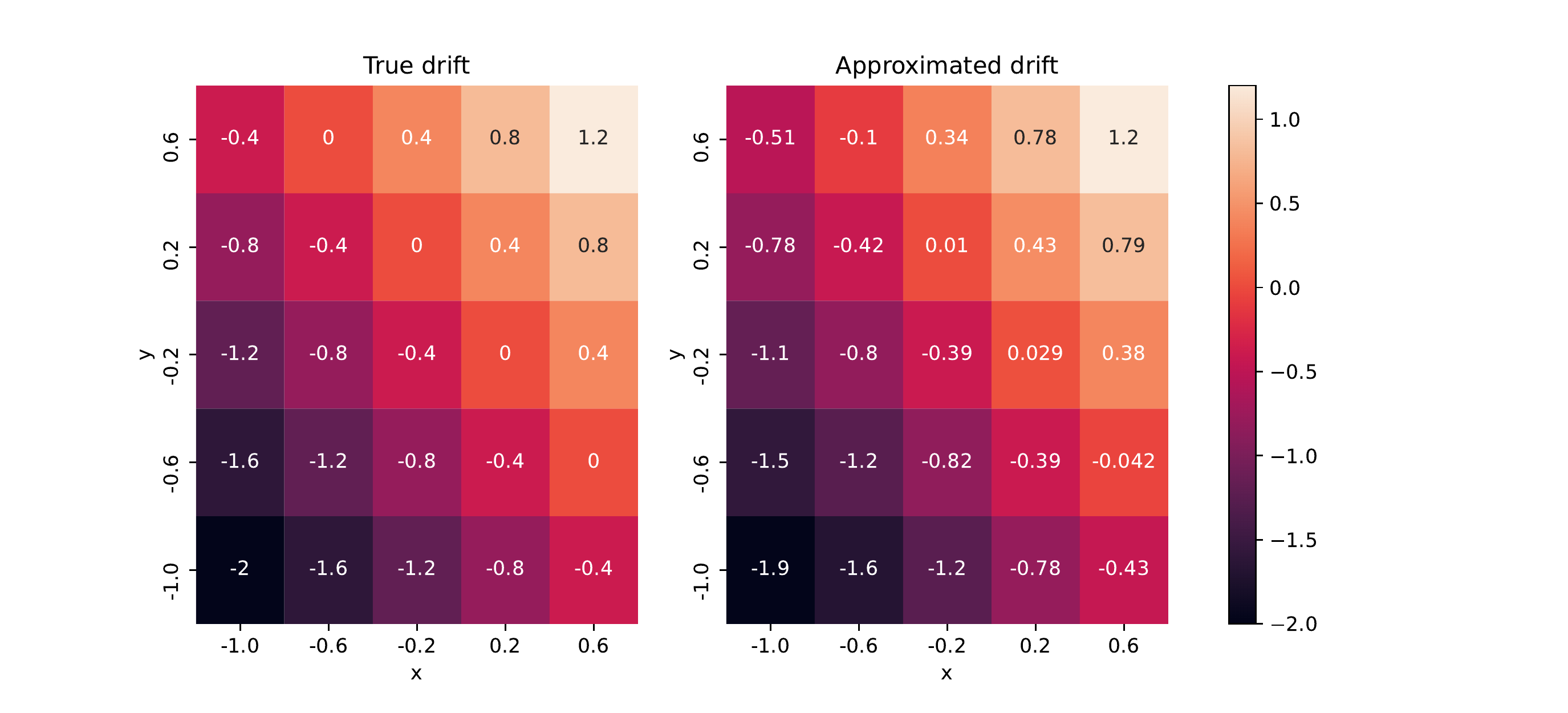}}
  \hspace{0.1in} 
  \subfigure[$f_2(X_t,Y_t)=4X_t-2Y_t$, (left) true drift coefficient, (right) approximated drift coefficient]{
    \label{fig:subfig:2d-lin-drift2} 
    \includegraphics[scale=0.38]{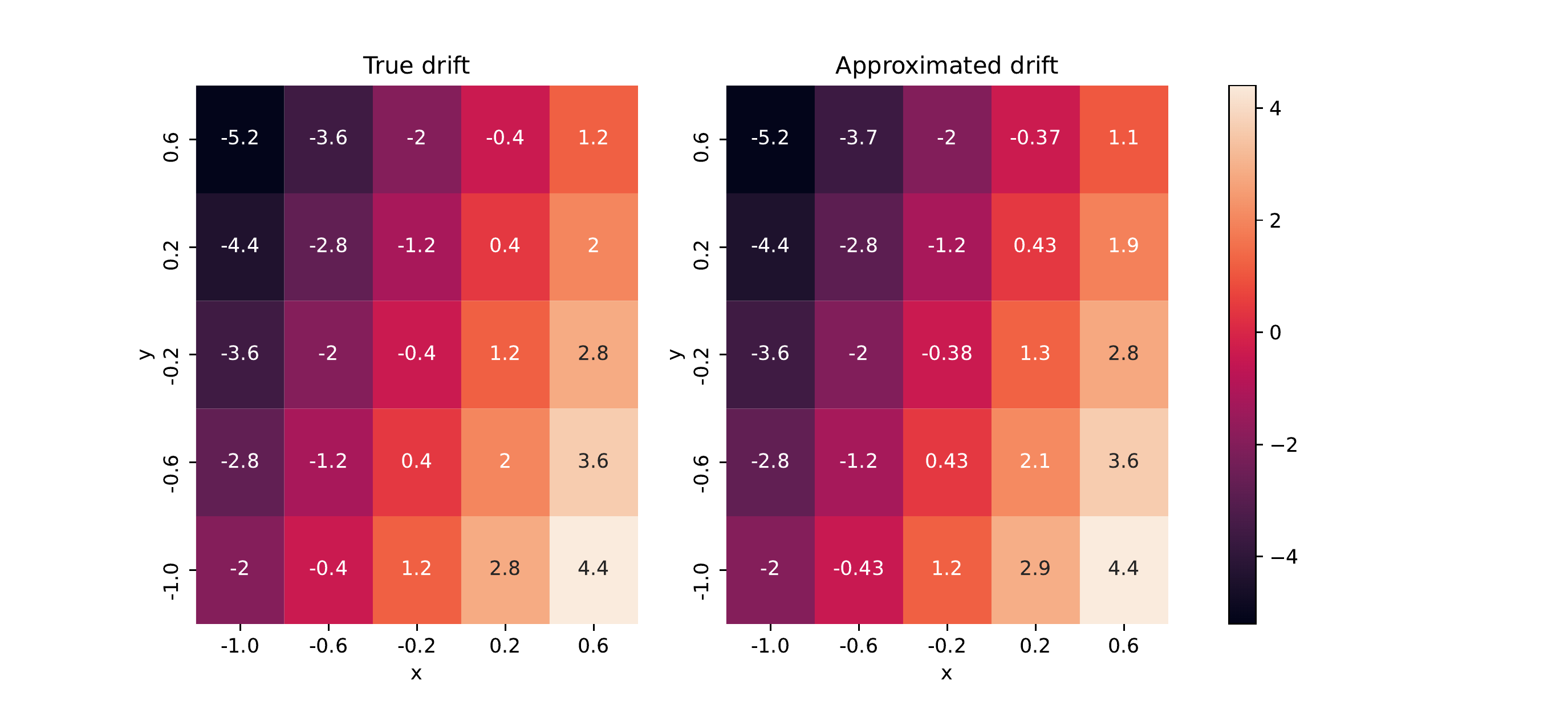}}
  \hspace{0.1in} 
  \subfigure[$g_1(X_t,Y_t)=0.5Y_t+1$, (left) true diffusivity coefficient, (right) approximated diffusivity coefficient]{
    \label{fig:subfig:2d-lin-diffusion1} 
    \includegraphics[scale=0.38]{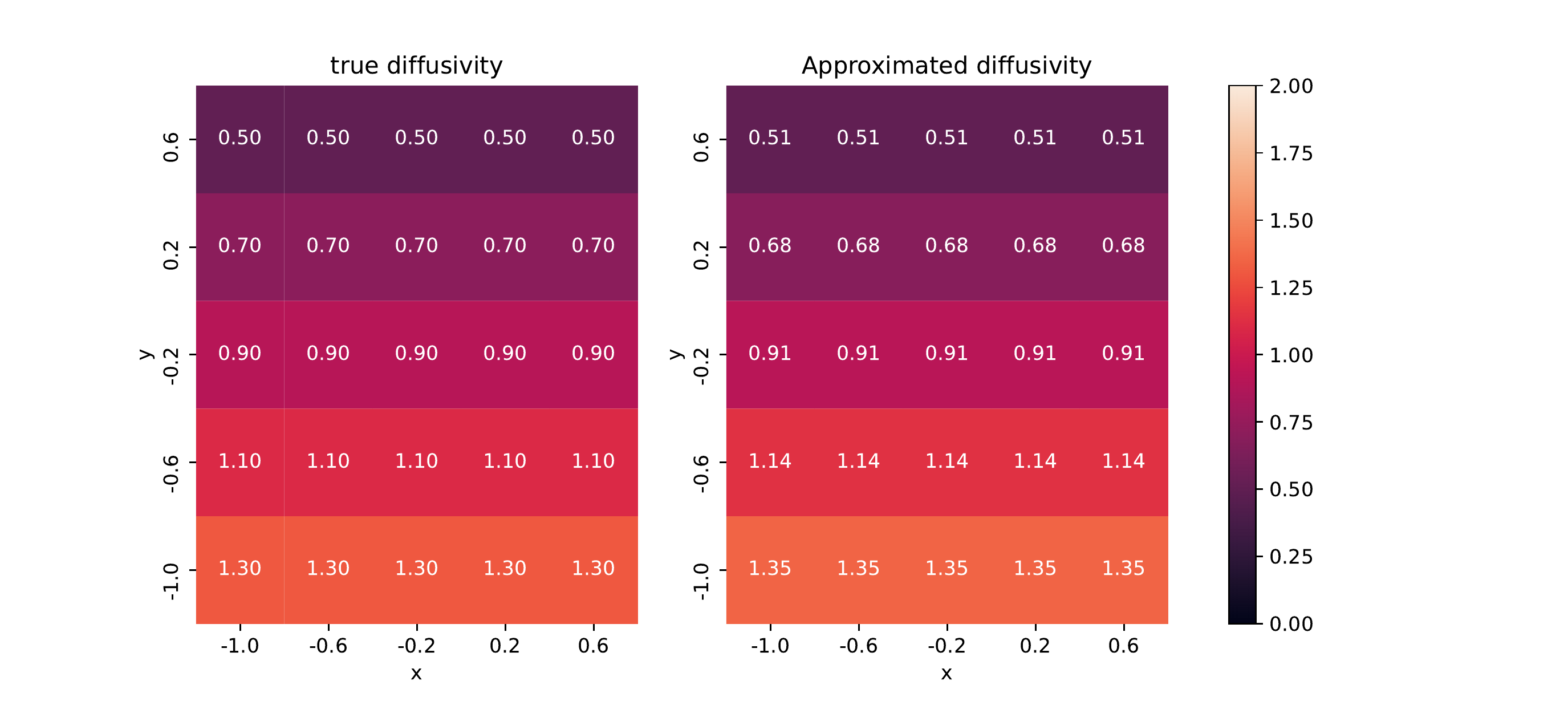}}
  \hspace{0.1in} 
  \subfigure[$g_2(X_t,Y_t)=0.5X_t+1$, (left) true diffusivity coefficient, (right) approximated diffusivity coefficient]{
    \label{fig:subfig:2d-lin-diffusion2} 
    \includegraphics[scale=0.38]{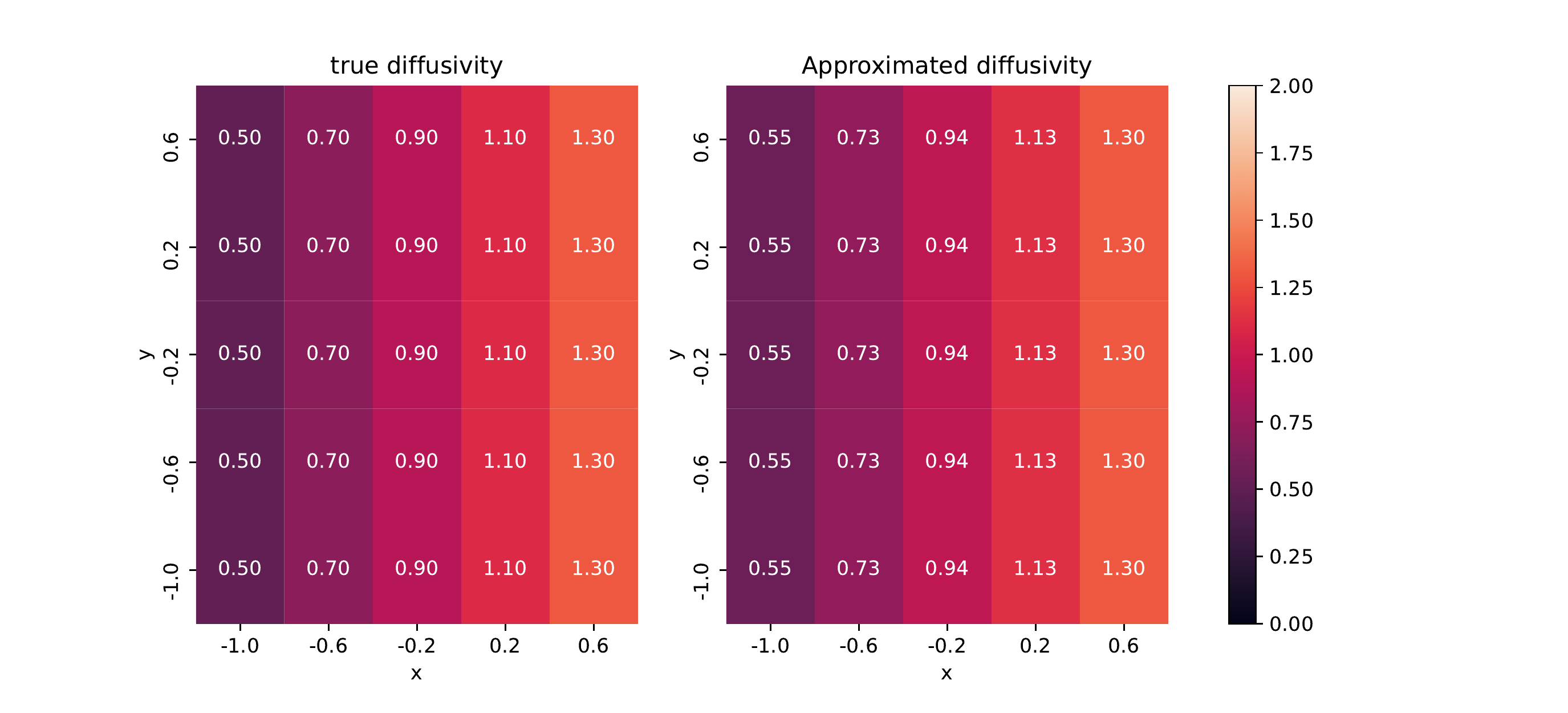}}
  \caption{The drift coefficients (true and estimated)  and diffusion coefficients (true and estimated) of the model (\ref{2d-lin-lin}), presented as heat maps.}
  \label{fig:l2d-in-lin MODEL} 
\end{figure}

\end{appendices}

\end{document}